%% file: root.tex
\documentclass{IEEEtran}
\pdfoutput=1
\usepackage{cite}
\usepackage{color}
\usepackage{graphicx}
\usepackage{epstopdf}
\usepackage[cmex10]{amsmath}
\def\Real{\mathbb{R}}
\usepackage{bbm}
\usepackage{algorithm,algpseudocode}

\usepackage{amsmath,amssymb}
\usepackage{lipsum}
\usepackage{comment}
\usepackage{array}
\input{mysymbol.sty}
\input{my_sections.tex}

\usepackage{amsthm,bm}
\usepackage{dsfont}
\usepackage{url}
\usepackage{threeparttable}

\usepackage{subfigure}
\usepackage{tikz}
\usepackage{pgfplots}
\usepgfplotslibrary{groupplots}
\usetikzlibrary{external}
\tikzsetexternalprefix{figures/}
\tikzset{external/optimize=false}

\definecolor{mygreen}{rgb}{0.10,0.50,0.10}

\usepackage{ifthen}
\newboolean{showcomments}
\setboolean{showcomments}{true}
\usepackage{todonotes}

\newcommand{\santiago}[1]{  \ifthenelse{\boolean{showcomments}}
{\todo[inline,color=cyan]{Santiago: #1}}{}}

\newcommand{\weiqin}[1]{  \ifthenelse{\boolean{showcomments}}
{\todo[inline,color=orange]{Weiqin: #1}}{}}

\newcommand{\reviewer}[1]{  \ifthenelse{\boolean{showcomments}}
{\todo[inline,color=red]{Reviewer: #1}}{}}

\newtheorem{proposition}{Proposition}

\newtheorem{theorem}{Theorem}
\newtheorem{definition}{Definition}
\newtheorem{lemma}{Lemma}
\newtheorem{corollary}{Corollary}

\theoremstyle{definition}

\author{Weiqin Chen$^\dagger$, 
Dharmashankar Subramanian$^\S$ and Santiago Paternain$^\dagger$
\thanks{$^\dagger$ Department of Electrical Computer and Systems Engineering, Renssealaer Polytechnic Institute. Email: chenw18@rpi.edu, paters@rpi.edu. $^\S$ IBM T.J. Watson Research Center. Email: dharmash@us.ibm.com
}}
\renewcommand{\comment}[1]{}
\newcolumntype{S}{>{\centering\arraybackslash} m{.10\linewidth} }
\newcolumntype{T}{>{\centering\arraybackslash} m{.30\linewidth} }

\title{Probabilistic Constraint for Safety-Critical Reinforcement Learning}

\usepackage{fancyhdr}
\begin{document}

\maketitle

\thispagestyle{fancy}

\begin{abstract}
In this paper, we consider the problem of learning safe policies for probabilistic-constrained reinforcement learning (RL). Specifically, a safe policy or controller is one that, with high probability, maintains the trajectory of the agent in a given safe set. We establish a connection between this probabilistic-constrained setting and the cumulative-constrained formulation that is frequently explored in the existing literature. We provide theoretical bounds elucidating that the probabilistic-constrained setting offers a better trade-off in terms of optimality and safety (constraint satisfaction). The challenge encountered when dealing with the probabilistic constraints, as explored in this work, arises from the absence of explicit expressions for their gradients. Our prior work provides such an explicit gradient expression for probabilistic constraints which we term Safe Policy Gradient-REINFORCE (SPG-REINFORCE). In this work, we provide an improved gradient SPG-Actor-Critic that leads to a lower variance than SPG-REINFORCE, which is substantiated by our theoretical results. A noteworthy aspect of both SPGs is their inherent algorithm independence, rendering them versatile for application across a range of policy-based algorithms. Furthermore, we propose a Safe Primal-Dual algorithm that can leverage both SPGs to learn safe policies. It is subsequently followed by theoretical analyses that encompass the convergence of the algorithm, as well as the near-optimality and feasibility on average. In addition, we test the proposed approaches by a series of empirical experiments. These experiments aim to examine and analyze the inherent trade-offs between the optimality and safety, and serve to substantiate the efficacy of two SPGs, as well as our theoretical contributions.
\end{abstract}

\input{introduction}
\input{problem_formulation}
\input{gradient}
\input{algorithms}
\input{numerical_results}
\input{conclusions}
\appendix
\input{appendix}

\bibliographystyle{IEEEtran}
\bibliography{bib}
% \vspace{-0.5cm}

\begin{IEEEbiography}[{\includegraphics[width=1in,height=1.25in,clip,keepaspectratio]{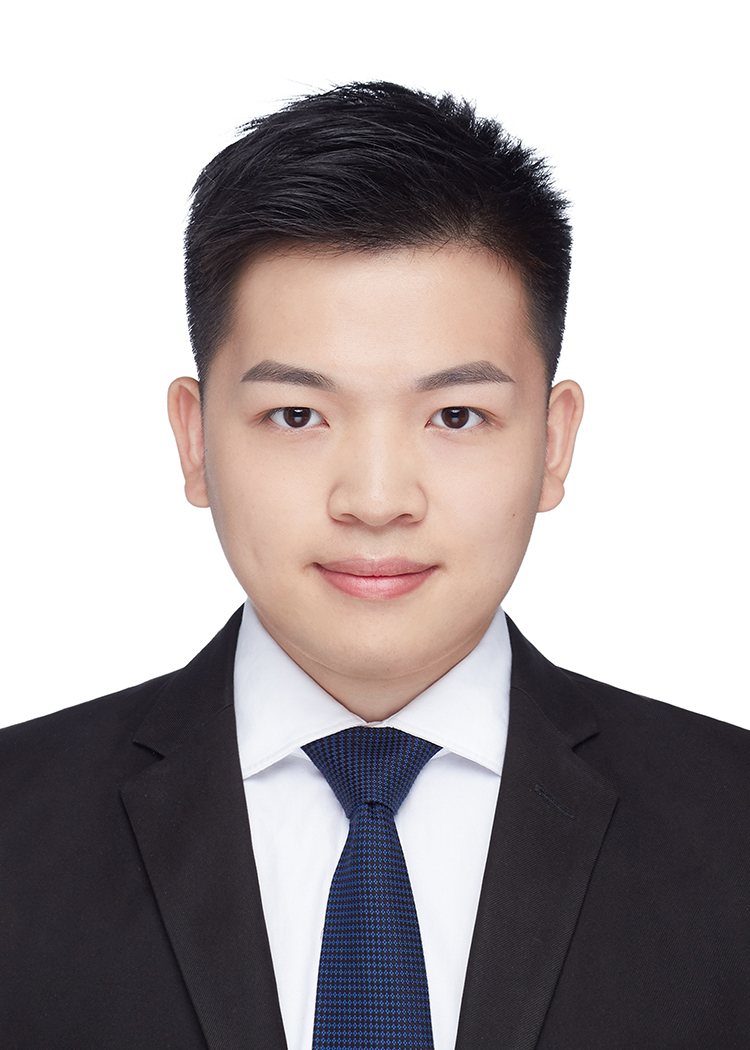}}]{Weiqin Chen} has been
working toward the Ph.D. degree in the Department of Electrical Computer and Systems Engineering at Rensselaer Polytechnic Institute since August 2021. Mr. Chen was the recipient of the Meritorious Winner Award in the Mathematical / Interdisciplinary Contest in Modeling in 2019 and Belsky Award for Computational Sciences and Engineering in 2023. His research interests include optimization, reinforcement learning, and LLM.
\end{IEEEbiography}

\vspace{-0.5cm}

\begin{IEEEbiography}[{\includegraphics[width=1in,height=1.25in,clip,keepaspectratio]{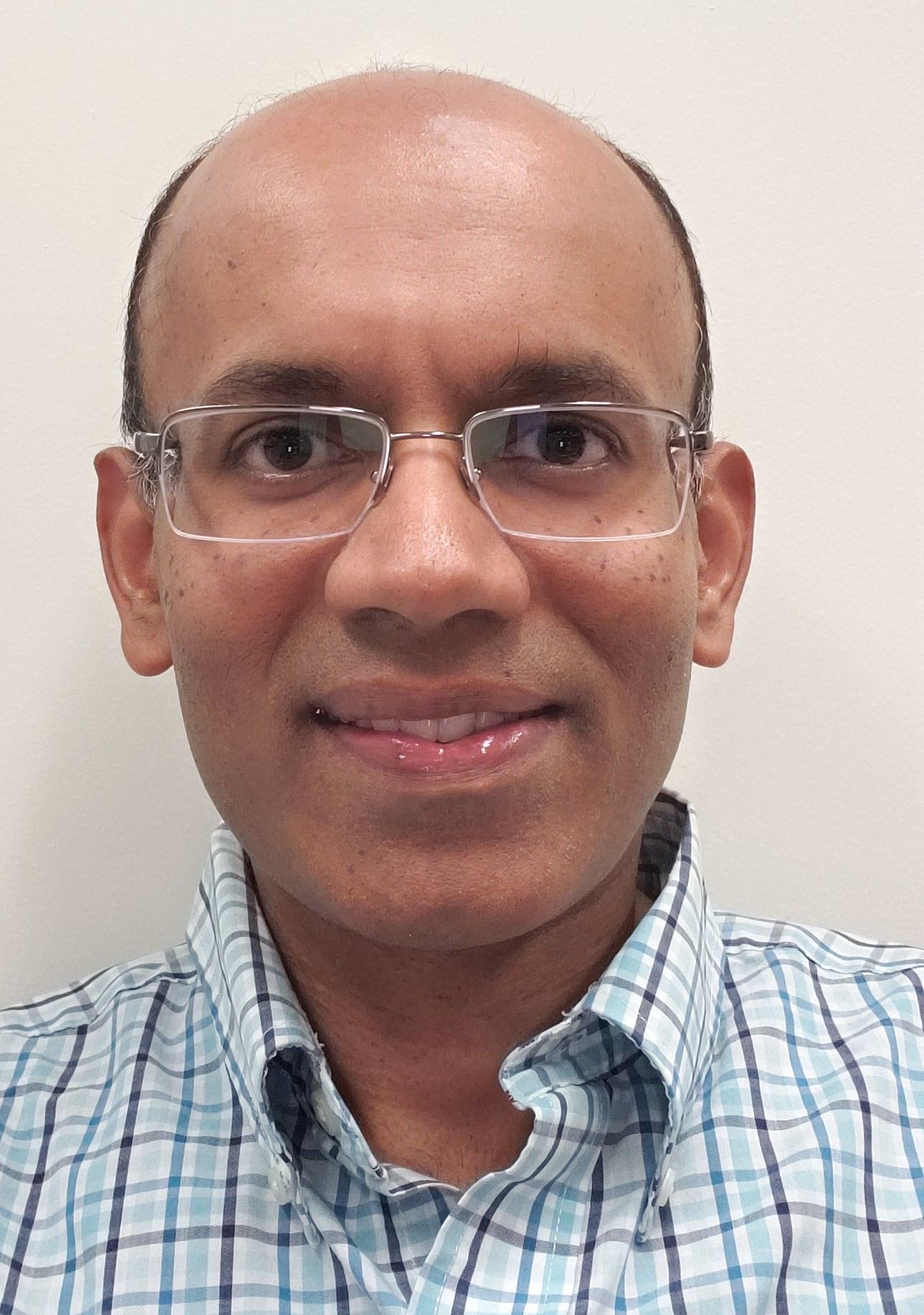}}]{Dharmashankar Subramanian} received the Ph.D. degree in chemical engineering from Purdue University, West Lafayette, IN, where he focused on combinatorial decision-making under uncertainty. He is currently a Principal Research Scientist \& Manager at IBM T.J. Watson Research Center, Yorktown Heights, NY. Prior to joining IBM Research, he was a Senior Research Scientist at Honeywell Labs, Minneapolis, MN. His research interests and expertise lie at the intersection of machine learning and optimal decision-making, and applications of mathematical modeling, and reinforcement learning in various domains for process control, optimization, risk analysis and alignment of large language models.
\end{IEEEbiography}

\vspace{-0.5cm}

\begin{IEEEbiography}[{\includegraphics[width=1in,height=1.25in,clip,keepaspectratio]{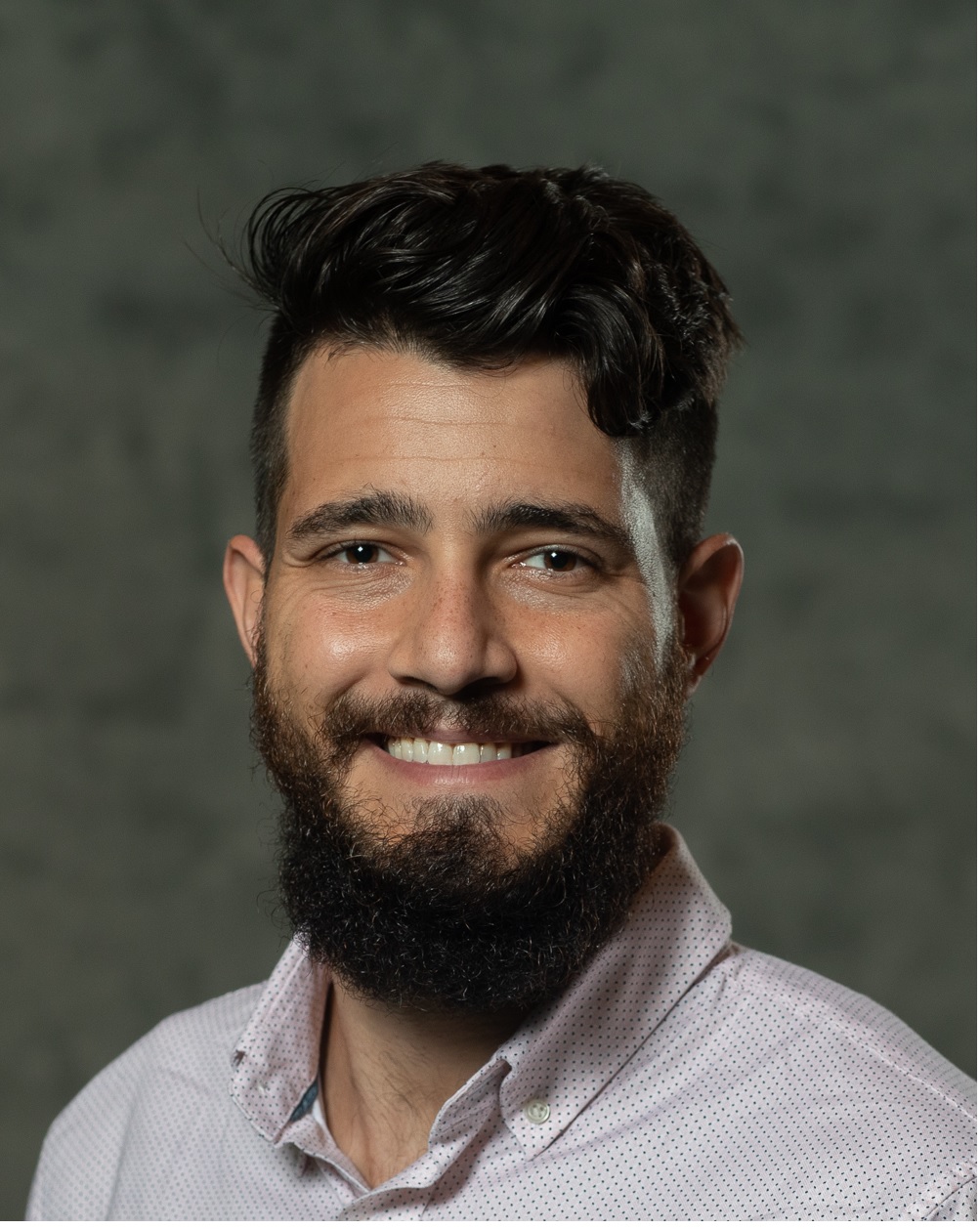}}]{Santiago Paternain} received the B.Sc. degree in electrical engineering from Universidad de
la República Oriental del Uruguay, Montevideo, Uruguay in 2012, the M.Sc. in Statistics from the Wharton School in 2018 and the Ph.D. in Electrical and Systems Engineering from the Department of Electrical and Systems Engineering, the University of Pennsylvania in 2018. He is currently an Assistant Professor in the Department of Electrical Computer and Systems Engineering at Rensselaer Polytechnic Institute. Prior to joining Rensselaer, Dr.
Paternain was a postdoctoral Researcher at the University of Pennsylvania. His research interests lie at the intersection of machine learning and control of dynamical systems. Dr. Paternain was the recipient of the 2017 CDC Best Student Paper Award and the 2019 Joseph and Rosaline Wolfe Best Doctoral Dissertation Award from the Electrical and Systems Engineering Department at the University of Pennsylvania.
\end{IEEEbiography}

\end{document}

%% file: my_sections.tex
\usepackage{needspace}

%% file: introduction.tex
%!TEX root = root.tex
\section{Introduction}\label{sec_intro}

Reinforcement learning (RL) has succeeded in solving a wide variety of sequential decision-making problems, such as playing video games~\cite{mnih2013playing}, control problems~\cite{shreve1978alternative}, robotic manipulation~\cite{levine2016end} and robot locomotion~\cite{duan2016benchmarking}. RL problems are generally formulated as Markov Decision Processes (MDPs). When dynamic models of MDPs are accessible, optimal \emph{policies} can be derived using dynamic programming~\cite{bertsekas1996neuro}. However, in cases where the underlying dynamics are unknown, the policy needs to be learned from system samples. Numerous RL algorithms~\cite{williams1992simple, lillicrap2015continuous, schulman2015trust, schulman2017proximal, haarnoja2018soft}, exhibiting diverse degrees of effectiveness, have been developed to tackle these problems. Nevertheless, in general, RL algorithms are only concerned with maximizing the expected cumulative reward~\cite{watkins1992q,sutton1999policy}, which may lead to unsafe behaviors~\cite{garcia2015comprehensive} in realistic domains.

Safety constitutes a foundational aspect in the design of control systems for physical entities. Specifically, controllers employed in power systems are meticulously crafted to avert voltage instabilities~\cite{van2000voltage}, which may lead to hazardous operational conditions. Similarly, in the realm of robot navigation, ensuring collision avoidance~\cite{kahn2018self} is essential for their proper functioning, and to ensure the preservation of human safety in the vicinity. Taking into account the safety requirements motivates the development of policy optimization under safety guarantees~\cite{geibel2006reinforcement,kadota2006discounted,chow2017risk}. Several approaches consider risk-aware objectives where the reward is modified to take into account the safety requirements~\cite{howard1972risk,sato2001td,geibel2005risk}. A limitation of these approaches is that reward shaping (the process of combining reward with safety requirements) is, in general, a time-consuming process of hyper-parameter tuning that requires human intervention and it is problem dependent~\cite{leike2017ai,mania2018simple}.

To mitigate this issue, a common approach is to employ the framework of Constrained MDPs (CMDPs)~\cite{altman1999constrained},
where additional expected cumulative (or average) cost needs to be maintained within a desired threshold. This framework has gained widespread adoption for inducing safe behaviors~\cite{borkar2005actor,bhatnagar2012online,liang2018accelerated, achiam2017constrained, tessler2018reward, yang2020projection, zhang2020first, liu2020ipo, shen2022penalized}. To solve these constrained optimization problems, regularization methods~\cite{censor1977pareto} and primal-dual algorithms~\cite{borkar2005actor,bhatnagar2012online,liang2018accelerated,achiam2017constrained, tessler2018reward} are generally considered.
In this setting, it is noteworthy that even safety violations in all trajectories are acceptable as long as the amount of violations does not exceed the desired threshold. This makes them often not suitable for safety-critical applications~\cite{cheng2019end, zhang2020cautious, corsi2021formal}. For instance, in the context of autonomous driving, even one single collision with obstacles or pedestrians is unacceptable.

A more suitable notion for safety-critical contexts, is to guarantee that the whole trajectory of the system remains within a set that is deemed to be safe. Ideally, one would like to achieve this goal for every possible trajectory. This being an ambitious goal, in this work we settle for high probability guarantees. We describe this setting in detail in Section \ref{sec_problem_formulation}. Problems with such probabilistic constraints have been considered in~\cite{geibel2006reinforcement,delage2010percentile}. The main challenge in solving them is that policy gradient-like expressions for the probabilistic constraints are not readily available. This, in turn, prevents from running classical policy-based algorithms such as REINFORCE~\cite{williams1992simple}, as well as state-of-the-art methods like  DDPG~\cite{lillicrap2015continuous}, PPO~\cite{schulman2017proximal}, SAC~\cite{haarnoja2018soft}, CQL~\cite{kumar2020conservative}, CSC~\cite{bharadhwaj2020conservative}, and constrained Q learning~\cite{kalweit2020deep} in the probabilistic-constrained RL setting. 
% \blue{Likewise, the cutting-edge algorithms conservative Q-learning (CQL)~\cite{kumar2020conservative} and conservative safety critics (CSC)~\cite{bharadhwaj2020conservative} do not apply to probabilistic-constrained RL. On the other hand, to the best of our knowledge CQL is better to be applied to offline RL~\cite{levine2020offline} and CSC is more suitable for safe exploration~\cite{garcia2015comprehensive} which are not the focus and interest of this work.} 
Perhaps, for this reason, cumulative constraints are normally considered in the literature~\cite{paternain2022safe,calvo2021towards}. In Section~\ref{sec_Safe_Reinforcement_Learning_Problems_Properties}, we delve into the interrelationships between probabilistic-constrained and cumulative-constrained settings in terms of the objective optimality and safety satisfaction. The cumulative constraint can be construed as a more lenient form of the probabilistic constraint, providing a relaxed framework for achieving desired outcomes. In addition, we establish theoretical bounds indicating that working directly with the probabilistic constraint provides advantages over the cumulative setting, i.e., an improved optimality-safety trade-off (Theorem \ref{theorem_P_star_Ptilder_star}).

Described in Section~\ref{Gradient of the Probabilistic Constraint}, our prior work~\cite{chen2023policy} provides, to the best of our knowledge, the first explicit expression for the gradient of the probabilistic constraint, which can be naturally applied to the aforementioned classical and state-of-the-art algorithms. We term the gradient established in \cite{chen2023policy} Safe Policy Gradient-REINFORCE (SPG-REINFORCE). Building on \cite{chen2023policy}, we here introduce an improved gradient, which we term SPG-Actor-Critic that exhibits reduced variance (Theorem \ref{theorem_lower_variance}). Section~\ref{Losses_of_Imposing_Relaxation} proposes a Safe Primal-Dual algorithm that can leverage both SPGs to learn safe policies. Furthermore, we establish the convergence of the algorithm, as well as the feasibility and near-optimality of the algorithm on average. Other than concluding remarks (Section~\ref{sec_conclusions}), this paper finishes with a series of numerical experiments in Section~\ref{Numerical_Experiments},
%\red{. We first consider a continuous navigation task that illustrates (i) the ability to implement Safe Primal-Dual algorithm that leverage both SPGs to train safe policies and (ii) the theoretical bounds established in Section~\ref{sec_Safe_Reinforcement_Learning_Problems_Properties} relating the problems with cumulative and probabilistic constraints. In particular, the latter traces a better optimality-safety trade-off. Subsequently, a lunar lander problem~\cite{brockman2016openai} showcases the successful application of SPGs in systems characterized by intricate dynamics.}
%
where we illustrate (i) the ability to learn safe policies with the Safe Primal-Dual algorithm leveraging both SPGs, (ii) the reduced variance of SPG-Actor-Critic, and (iii) the improved optimality-safety trade-offs that probabilistic constraints provide over cumulative constraints.

%% file: problem_formulation.tex
\section{Problem Formulation}\label{sec_problem_formulation}

% \red{clarify the notation

% what is $\mu$ sample space, event space. It is not comprehensive}
In this work, we consider the problem of finding optimal policies for Markov Decision Processes (MDPs) under probabilistic safety guarantees. In particular, we are interested in situations where the state transition distributions are unknown, and thus the policies need to be computed from system samples. An MDP~\cite{sutton2018reinforcement} is defined by a tuple ($\mathcal{S}, \mathcal{A}, r, \mathbb{P}, \mu, T$), where $\mathcal{S}$ is the state space, $\mathcal{A}$ is the action space, $r: \mathcal{S} \times \mathcal{A} \to \Real$ is the reward function describing the quality of the decision. For any $\hat{\mathcal{S}} \subset \mathcal{S}, s_t \in \mathcal{S}, a_t \in \mathcal{A}, t\in\left\{0,1,\ldots, T\right\}$,  {$\mathbb{P}_{s_t \to s_{t+1}}^{a_t} (\hat{\mathcal{S}}):=\mathbb{P}(s_{t+1} \in \hat{\mathcal{S}} \mid s_t, a_t)$ is the transition probability describing the dynamics of the system, $\mu (\hat{\mathcal{S}}):= \mathbb{P}(s \in \hat{\mathcal{S})}$ is the initial state distribution,} and $T$ is the time horizon. The state and action at time $t\in\left\{0,1,\ldots, T\right\}$ are random variables denoted respectively by $S_t$ and $A_t$. A \emph{policy} is a conditional distribution $\pi_\theta (a|s)$ parameterized by $\theta \in \Real^d$ (for instance the weights and biases of neural networks), from which the agent draws action $a \in \mathcal{A}$ when in the corresponding state $s \in \mathcal{S}$. In the context of MDPs the objective is to find a policy that maximizes the value function which is defined as 
\begin{equation}\label{eqn_val_func_finite}
    V (\theta) =  \mathbb{E}_{\mathbf{a} \sim \pi_\theta (\mathbf{a}|\mathbf{s}), S_0 \sim \mu} \left[\sum\limits_{t=0}^{T} r(S_t, A_t)\right],
\end{equation}
where $\mathbf{a}$ and $\mathbf{s}$ denote the sequences of actions and states for the whole episode, this is, from time $t=0$ to $t=T$. Subscripts of the expectation are omitted in the remainder of this paper for notation simplicity.

% For any state $S_t \in \mathcal{S}$, the selected action $A_t \in \mathcal{A}$ drives the agent to the next state $S_{t+1} \in \mathcal{S}$ in accordance with the transition dynamics of the system demonstrated by a conditional probability $\mathbb{P}_{S_t \rightarrow S_{t+1}}^{A_t} (\hat{\mathcal{S}}) := \mathbb{P}(S_{t+1} \in \hat{\mathcal{S}} | S_t, A_t)$, where $\hat{\mathcal{S}} \subset \mathcal{S}$. By virtue of the Markov property of the system $\mathbb{P}(S_{t+1} \in \hat{\mathcal{S}} | (S_i, A_i), \forall i \leq t )=\mathbb{P}(S_{t+1} \in \hat{\mathcal{S}} | S_t, A_t)$, it is known as a MDP.

By attempting to maximize a single objective, policies may be unsafe or result in risky behaviors. Thus, we impose requirements in terms of probabilistic safety to overcome this limitation. We formally define this notion next. 
\begin{definition}
\label{definition_safety}
A policy $\pi_\theta$ is $(1-\delta)$-safe for the set $\mathcal{S}_\text{safe} \subset \mathcal{S}$ if and only if $\mathbb{P} \left(\cap_{t=0}^{T} \{ S_t \in \mathcal{S}_\text{safe}\} |\pi_\theta \right) \geq 1-\delta$.
\end{definition}
In the previous definition  $\cap_{t=0}^T\left\{S_t\in \ccalS_{safe}\right\}$ refers to the intersection of events of the form $\left\{S_t\in \ccalS_{safe}\right\}$ for all times. This is, we require the state $S_t$ to belong to the safe set for all times $t=0,\ldots T$ with high probability.
% \blue{Notice that
% $\mathbb{P} \left(\cap_{t=0}^{T} \{ S_t \in \mathcal{S}_\text{safe}\} |\pi_\theta \right)$ in the previous definition denotes the probability of the whole episode being safe given the policy $\pi_\theta$, i.e., every state $S_t$ (from $t=0$ to $t=T$) remains within the safe set $\mathcal{S}_\text{safe}$.}
Under Definition~\ref{definition_safety}, we formulate the probabilistic safe RL problem as the following constrained optimization problem
\begin{align}\label{eqn_problem1}
 P^\star = &\max\limits_{\theta \in \Real^d} \, V(\theta)      \nonumber \\
 &\text{s.t.} \quad \mathbb{P} \left(\bigcap\limits_{t=0}^{T} \{ S_t \in \mathcal{S}_\text{safe}\} |\pi_\theta \right) \geq 1-\delta. 
\end{align}

To solve problem~\eqref{eqn_problem1}, it is conceivable to employ gradient-based methods e.g., regularization~\cite{censor1977pareto} and primal-dual~\cite{arrow1958studies} to achieve local optimal solutions. For instance, consider the regularization method with a fixed penalty. This is, for~$\lambda>0$ we formulate the following \emph{unconstrained} problem as an approximation to the \emph{constrained} problem \eqref{eqn_problem1}
% \begin{align}\label{eqn_agmt_obj}
%  \mathbb{E} \left[\sum\limits_{t=0}^{T} r(S_t, A_t) | \pi_\theta \right] \! + \! \lambda \! \left(\!  \mathbb{P}  \left(\bigcap\limits_{t=0}^{T} \{ S_t \in \mathcal{S}_\text{safe}\} |\pi_\theta \right) \!- \! (1-\delta) \! \right),
% \end{align}
\begin{align}\label{eqn_agmt_obj}
 \max\limits_{\theta \in \Real^d} \, V(\theta)  +  \lambda  \left( \mathbb{P}  \left(\bigcap\limits_{t=0}^{T} \{ S_t \in \mathcal{S}_\text{safe}\} |\pi_\theta \right) -  (1-\delta) \right).
\end{align}
It is important to point out that, in general, there is no guarantee that a fixed coefficient $\lambda$ achieves the same solution as \eqref{eqn_problem1} (an exception is, for instance, in cases where \eqref{eqn_problem1} is convex~\cite{boyd2004convex}). However, $\lambda$ trades-off optimality and safety. Indeed, for large values of $\lambda$ solutions to \eqref{eqn_agmt_obj} will prioritize safe behaviors, whereas for small values of $\lambda$ the solutions will focus on maximizing the value function. Alternatively, one approach to automatically search appropriate values of $\lambda$ is to use iterative methods such as the primal-dual algorithm (See Section \ref{Losses_of_Imposing_Relaxation}).

% {\color{red}{Similarly, the primal-dual method operates by updating $\lambda$ at each iteration. In this context, we present a novel approach called Safe Primal-Dual algorithm, which is detailed in Section~\ref{Losses_of_Imposing_Relaxation}. To validate the effectiveness of this algorithm, we conduct numerical experiments, and the results are discussed in Section~\ref{Numerical_Experiments}.}}

To solve problem~\eqref{eqn_agmt_obj} locally, gradient ascent~\cite{bertsekas1997nonlinear} or its stochastic versions can be used. Note that the gradient of $V(\theta)$ in~\eqref{eqn_agmt_obj} can be computed using the Policy Gradient Theorem~\cite{sutton1999policy}. Nevertheless, the lack of an expression for the gradient of the probabilistic safety, i.e., $\nabla_\theta \mathbb{P} \left(\cap_{t=0}^{T} \{ S_t \in \mathcal{S}_\text{safe}\} |\pi_\theta \right)$ prevents us from applying this family of methods to solve \eqref{eqn_agmt_obj}. 

It is worth pointing out that other state-of-the-art algorithms that are developed to solve constrained RL problems, e.g., CPO~\cite{achiam2017constrained}, RCPO~\cite{tessler2018reward}, PCPO~\cite{yang2020projection}, FOCOPS~\cite{zhang2020first} also rely on computing the gradients of objective functions and constraints. As such, it is not surprising that when considering safe RL formulations, instead of dealing with problems of the form \eqref{eqn_problem1} they consider cumulative constraints. In these problems, an auxiliary reward function $r_c:\ccalS\times\ccalA\to \mathbb{R}$ is defined, and the following problem is formulated
\begin{align}\label{eqn_problem2}
 \tilde{P}^\star (\xi) = &\max\limits_{\theta \in \Real^d} \, V(\theta)      \nonumber \\
 &\text{s.t.} \quad V_c(\theta):=~\mathbb{E} \left[\sum\limits_{t=0}^{T} r_c(S_t, A_t)\right]\geq \xi,
\end{align}
where $\xi$ is a hyper-parameter that induces different levels of safety. Problems of the form \eqref{eqn_problem2} are known as CMDPs~\cite{altman1999constrained} and they can be tackled through classic policy gradient, e.g.,~\cite{sutton1999policy}. When formulating safe RL as a CMDP, the function $r_c$ is designed so that it induces safe behaviors when attaining large values. This results in a problem-dependent design and in general, it necessitates a time-consuming process of hyper-parameter tuning. Moreover, the probabilistic safety as in Definition \ref{definition_safety} is not guaranteed by formulation \eqref{eqn_problem2}. Both of these limitations are avoided under formulation \eqref{eqn_problem1}, which we focus on in this work.

While problems \eqref{eqn_problem1} and \eqref{eqn_problem2} may appear distinct, they share a significant connection. The next section delves into this relationship in terms of objective optimality. Specifically, we draw insights from~\cite{paternain2022safe,wagener2021safe} to establish explicit bounds on safety guarantees and optimal values for problems \eqref{eqn_problem1} and \eqref{eqn_problem2}.

%%%%%%%%%%%%%%%%%%%%%%%%%%%%%
%%%%%%%%%%%%%%%%%%%%%%%%%%%%%
%%%%%%%%%%%%%%%%%%%%%%%%%%%%%
%%%%%%%%%%%%%%%%%%%%%%%%%%%%%
%%%%%%%%%%%%%%%%%%%%%%%%%%%%%
%%%%%%%%%%%%%%%%%%%%%%%%%%%%%
%%%%%%%%%%%%%%%%%%%%%%%%%%%%%
%%%%%%%%%%%%%%%%%%%%%%%%%%%%%
%%%%%%%%%%%%%%%%%%%%%%%%%%%%%
%%%%%%%%%%%%%%%%%%%%%%%%%%%%%
%%%%%%%%%%%%%%%%%%%%%%%%%%%%%
%%%%%%%%%%%%%%%%%%%%%%%%%%%%%
\section{Properties of Safe Reinforcement Learning}
\label{sec_Safe_Reinforcement_Learning_Problems_Properties}
In the context of guaranteeing that the agent remains in a safe subset of the state space, a possibility~\cite{geibel2005risk} for choosing $r_c(S_t, A_t)$ in \eqref{eqn_problem2} is $r_c(S_t,A_t) = \mathbbm{1}(S_t \in \mathcal{S}_\text{safe}) / (T+1)$. In this case, the cumulative safety constraint is related to the function 
\begin{equation}\label{eqn_cumulative_safety}
V_c(\theta)= \mathbb{E} \left[ \frac{1}{T+1}  \sum\limits_{t=0}^{T} \mathbbm{1} ( S_t \in \mathcal{S}_\text{safe}) |\pi_\theta \right].
\end{equation}
The advantage of selecting the indicator function to address safety is that it is problem independent. In addition, by selecting an appropriate level of constraint satisfaction one can guarantee safe policies in the sense of Definition~\ref{definition_safety}. To be more precise, replace the left-hand side of the constraint of problem~\eqref{eqn_problem2} with \eqref{eqn_cumulative_safety} and set $\xi$ to be $1-\delta/(T+1)$ 
\begin{align}\label{eqn_problem2_mirror}
&\hat{P}^\star = \max\limits_{\theta \in \Real^d} \, V(\theta) \nonumber \\
& \text{s.t.} \quad \mathbb{E} \left[\frac{1}{T+1} \sum\limits_{t=0}^{T} \mathbbm{1} \left( S_t \in \mathcal{S}_\text{safe} \right) |\pi_\theta \right] \geq 1-\frac{\delta}{T+1}.
\end{align}
The formulation presented here guarantees that any $\theta$ that satisfies the constraint in \eqref{eqn_problem2_mirror} is guaranteed to be safe as in Definition~\ref{definition_safety} with probability $1-\delta$. We formally state this claim in the following proposition.

\begin{proposition}[\cite{paternain2022safe}, Theorem 1]
\label{proposition_U_cP_1-delta}
Denote by $\tilde{\theta}$ a feasible solution to problem \eqref{eqn_problem2_mirror}. Then, $\tilde{\theta}$ is a feasible solution to problem \eqref{eqn_problem1} as well, i.e., the policy induced by $\tilde{\theta}$ guarantees safety in the sense of  Definition~\ref{definition_safety}.
\end{proposition}

Although Proposition~\ref{proposition_U_cP_1-delta} confirms that any feasible solution to problem~\eqref{eqn_problem2_mirror} is also feasible for problem~\eqref{eqn_problem1}, the converse is not true. This indicates that $\hat{P}^\star$ in \eqref{eqn_problem2_mirror} is smaller than or equal to $P^\star$ in \eqref{eqn_problem1}. We formalize this claim in the next Theorem. We also establish an upper bound for the latter in terms of the former. Before doing so, we recall the definition of the optimal Lagrange multiplier associated with \eqref{eqn_problem2_mirror}. This important quantity will be used to state and prove several of the theoretical developments that follow. Let us start by defining the dual function associated with \eqref{eqn_problem2_mirror}
\begin{align}\label{eqn_dual_function}
 \hat{d}(\lambda) = \max\limits_{\theta \in \Real^d} \, V(\theta) + \lambda (V_c(\theta) - (1-\delta/(T+1))),
\end{align}
where $\lambda \geq 0$. The dual function provides an upper bound on problem~\eqref{eqn_problem2_mirror}~\cite[Chapter 5]{boyd2004convex}. Thus in general, one is interested in finding the $\lambda$ that provides the tightest of the upper bounds%. The $\hat{\lambda}$ where this is achieved is termed the optimal Lagrange multiplier
\begin{align}\label{eqn_dual_solution}
   \hat{\lambda}^\star = \argmin\limits_{\lambda \in \Real_+} \, \hat{d}(\lambda).
\end{align}
Having defined $\hat{\lambda}^\star$, we are in the stage of stating the bounds between the optimal values of \eqref{eqn_problem1} and \eqref{eqn_problem2_mirror}. This is the subject of the following theorem.

\begin{theorem}\label{theorem_P_star_Ptilder_star}
Let $\mathcal{P}_s(\mathcal{A})$ be the stochastic kernels with source in the state space $\mathcal{S}$ and target the action space $\mathcal{A}$. Assume that for any policy $\pi\in\mathcal{P}(s)$ and for all $(s,a)\in\mathcal{S}\times \mathcal{A}$ there exists a parameter $\theta \in \Real^d$ such that $\pi_\theta (a|s) = \pi (a|s)$. Consider optimal values $P^\star$ and $\hat{P}^\star$ in problems \eqref{eqn_problem1} and \eqref{eqn_problem2_mirror}, as well as the safety level $\delta$ and time horizon $T$. Denote by $\hat{\lambda}^\star$ the solution to the dual problem associated with problem~\eqref{eqn_problem2_mirror}, as defined in \eqref{eqn_dual_solution}. Then it holds that 
\begin{align}
\label{upper_lower_bound}
    \hat{P}^\star + \hat{\lambda}^\star \delta \frac{ T}{T+1}
    \geq P^\star\geq \hat{P}^\star.
\end{align}
\end{theorem}
\begin{proof}
We start by proving the rightmost inequality. Denote by $\tilde{\theta}^\dagger$ the optimal solution to problem~\eqref{eqn_problem2_mirror}. By virtue of Proposition~\ref{proposition_U_cP_1-delta} the policy $\pi_{\tilde{\theta}^\dagger}$ is $(1-\delta)$-safe in the sense of Definition~\ref{definition_safety}. Consequently, it is a feasible solution to problem~\eqref{eqn_problem1}. It follows, by definition of the optimal solution to problem \eqref{eqn_problem1}, that  $P^\star \geq V(\tilde{\theta}^\dagger) = \hat{P}^\star$. 

Having established the rightmost inequality of the claim, we set our focus on proving the leftmost inequality. To do so, consider the following perturbation to problem \eqref{eqn_problem2_mirror}
\begin{align}\label{eqn_problem_mirrormirror}
&\bar{P}^\star = \max\limits_{\theta \in \Real^d} \, V(\theta) \nonumber \\
& \text{s.t.} \quad \mathbb{E} \left[\frac{1}{T+1} \sum\limits_{t=0}^{T} \mathbbm{1} \left( S_t \in \mathcal{S}_\text{safe} \right) |\pi_\theta \right] \geq 1-\delta.
\end{align}
Note that problem~\eqref{eqn_problem_mirrormirror} is the same problem as \eqref{eqn_problem2_mirror} with a looser constraint. Applying Lemma~\ref{lemma_bound_zero_duality_gap} in Appendix~\ref{appendix_proposition_figure} to \eqref{eqn_problem2_mirror} and \eqref{eqn_problem_mirrormirror} yields the following relationship between their optimal values
\begin{align}\label{eqn_concavity}
    \bar{P}^\star \leq \hat{P}^\star + \hat{\lambda}^\star \delta\frac{T}{T+1}.
\end{align}
Analogous to the proof of the rightmost inequality in \eqref{upper_lower_bound}, employing Lemma~\ref{proposition_figure} in Appendix~\ref{appendix_proposition_figure} yields $\bar{P}^\star \geq P^\star$. Combining this fact with \eqref{eqn_concavity} completes the proof of Theorem~\ref{theorem_P_star_Ptilder_star}.
\end{proof}

It is important to highlight that Theorem~\ref{theorem_P_star_Ptilder_star} relies on the assumption of a universal parametrization. This fact is used in Lemma \ref{lemma_bound_zero_duality_gap} to establish the inequality \eqref{eqn_concavity}. A similar result can be stated if the policy belongs to a class of function approximators known as $\epsilon$-universal, e.g., radial basis functions networks~\cite{park1991universal}, reproducing kernel Hilbert spaces~\cite{sriperumbudur2010relation}, and deep neural networks~\cite{hornik1989multilayer}. In such cases, Theorem~\ref{theorem_P_star_Ptilder_star} remains valid, albeit with the introduction of an error to the bound in the leftmost inequality of \eqref{upper_lower_bound}. This error can be arbitrarily small by increasing the dimension of the parameterization. We state the result in this form for simplicity in the exposition.

Having clarified the assumption, we next focus on on the claim of Theorem~\ref{theorem_P_star_Ptilder_star}. The result provides lower and upper bounds on problem \eqref{eqn_problem1} that depends on the the optimal value of problem \eqref{eqn_problem2_mirror}, the level of safety $\delta$ and its optimal Lagrange multiplier $\hat{\lambda}^\star$. Note that for large horizons the optimality gap between the two problems can be approximated by $\hat{\lambda}^\star \delta$. This suggests that increasing the safety requirement, i.e., $\delta$ approaching zero would make the two problems equivalent. In fact, when $\delta = 0$ the two problems are equivalent. A caveat is that the Lagrange multiplier in turn depends on $\delta$. An interpretation of Lagrange multipliers is that they provide a measurement of how hard it is to satisfy the constraint in \eqref{eqn_problem2_mirror}. Thus, for systems, where satisfying the constraint is easier (smaller Lagrange multiplier), the bound between \eqref{eqn_problem1} and \eqref{eqn_problem2_mirror} is tighter. This interplay between the hardness of the problem and the bound is demonstrated numerically in Section \ref{Numerical_Experiments}.

% \red{
% Nevertheless, it is clear that the upper bound of $P^\star$ in \eqref{upper_lower_bound} is problem dependent. More specifically, the upper bound depends on the time horizon $T$, safety level $\delta$, and \red{$\hat{\lambda}^\star$--the solution to the dual problem associated with problem~\eqref{eqn_problem2_mirror}.}}
% \blue{I don't understand this comment}
% \red{Consequently, an increasing tight requirement on the right hand side of \eqref{eqn_problem2_mirror} is required to maintain the same safety guarantees as that in \eqref{eqn_problem1}. As intuitively shown in Fig.~\ref{fig_three_constraints}, this will result in the shrinkage of the feasible set as compared to problem \eqref{eqn_problem1} so that one could obtain only sub-optimal solutions to \eqref{eqn_problem1}.}

Notwithstanding the important observations made, solving the non-convex constrained optimization problem \eqref{eqn_problem1} remains challenging. One may instead attempt to solve its dual relaxation. In Corollary \ref{corollary_dual_bound} we characterize the duality gap between the optimal value of the dual problem and the primal problem. To do so, we define the dual function associated with \eqref{eqn_problem1} for any $\lambda > 0$, similar to \eqref{eqn_dual_function},
\begin{align}\label{eqn_dual_funtion2}
     d(\lambda) = \max\limits_{\theta \in \Real^d} \, V(\theta) \!+\! \lambda \!\left(\!\mathbb{P} \left(\bigcap\limits_{t=0}^{T} \{ S_t \in \mathcal{S}_\text{safe}\} |\pi_\theta \! \right) \! - \! (1-\delta) \! \right).
\end{align}
% Subsequently, akin to \eqref{eqn_dual_solution}, the optimal Lagrange multiplier is given by 
% \begin{align}\label{eqn_dual_solution2}
%     \lambda^\star = \argmin\limits_{\lambda \in \Real_+} \, d(\lambda).
% \end{align}
Subsequently, the dual optimum is defined by
\begin{align}\label{eqn_dual_optimum}
     {D}^\star = \min\limits_{\lambda \in \Real_+} d(\lambda).
\end{align}
We are now in conditions to state and prove a bound regarding the duality gap for problems \eqref{eqn_problem1} and \eqref{eqn_dual_optimum}.
\begin{corollary}\label{corollary_dual_bound}
Let hypotheses of Theorem 1 hold. Let $D^\star$ be the dual optimal value defined in \eqref{eqn_dual_optimum}. It holds that
\begin{equation}\label{corrolary_primal_dual}
   P^\star +\hat{\lambda}^\star\delta \frac{T}{T+1} \geq  {D}^\star \geq P^\star.
\end{equation}
\end{corollary}
\begin{proof}
The rightmost inequality follows directly from the property of \emph{weak duality}~{\cite[Equation 5.23]{boyd2004convex}}. We then focus on the proof of the leftmost inequality. For any $\lambda > 0$, combining \eqref{eqn_dual_optimum} with \eqref{eqn_dual_funtion2} yields
\begin{align}\label{eqn_dual_optimum2}
    {D}^\star \leq V(\theta^\star) \!+\! \lambda \!\left(\!\mathbb{P} \left(\bigcap\limits_{t=0}^{T} \{ S_t \in \mathcal{S}_\text{safe}\} |\pi_{\theta^\star} \! \right) \! - \! (1-\delta) \! \right),
\end{align}
where $\theta^\star$ denotes the maximizer of the Lagrangian \eqref{eqn_dual_funtion2}. Combining with \eqref{eqn_lemma1_aux2} the previous expression reduces to
\begin{align}
    {D}^\star \! \leq \! V(\theta^\star) \!+\! \lambda \!\!\left(\!\! \frac{1}{T+1} \mathbb{E} \! \left[\sum\limits_{t=0}^{T} \! \mathbbm{1} \left( S_t \in \mathcal{S}_\text{safe} \right) |\pi_{\theta^\star} \!\! \right]  \! - \! (1-\delta) \! \!\right).
\end{align}
In particular, the previous inequality holds for $\bar{\lambda}^\star$, which is the solution to the dual problem associated with problem \eqref{eqn_problem_mirrormirror}. Then it follows that
\begin{align}\label{eqn_dual_inequality}
    {D}^\star \! \leq \! V(\theta^\star) + \bar{\lambda}^\star \!\!\left(\!\! \frac{1}{T+1} \mathbb{E} \! \left[\sum\limits_{t=0}^{T} \! \mathbbm{1} \left( S_t \in \mathcal{S}_\text{safe} \right) |\pi_{\theta^\star} \!\! \right]  \! - \! (1-\delta) \! \!\right).
\end{align}
It is crucial to point out that $\theta^\star$ may not be the primal maximizer of the Lagrangian of problem \eqref{eqn_problem_mirrormirror}, which is denoted by $\bar{\theta}^\star$. Indeed, $\bar{\theta}^\star$ maximizes the right-hand side of \eqref{eqn_dual_inequality} for a fixed $\bar{\lambda}^\star$ thus yielding
\begin{align}
     {D}^\star \! &\leq \! V(\bar{\theta}^\star) \!+\! \bar{\lambda}^\star \!\!\left(\!\! \frac{1}{T+1} \mathbb{E} \! \left[\sum\limits_{t=0}^{T} \! \mathbbm{1} \left( S_t \in \mathcal{S}_\text{safe} \right) |\pi_{\bar{\theta}^\star} \!\! \right]  \! - \! (1-\delta) \! \!\right) \nonumber \\
     &=\bar{P}^\star,
\end{align}
where the last equation follows from the zero duality gap~\cite{paternain2019constrained} of problem \eqref{eqn_problem_mirrormirror}. Utilizing  \eqref{eqn_concavity} and the rightmost inequality of \eqref{upper_lower_bound} ${D}^\star$ can be further upper bounded as
\begin{align}
    {D}^\star \leq P^\star + \hat{\lambda}^\star \frac{\delta T}{T+1}.
\end{align}
This completes the proof of Corollary~\ref{corollary_dual_bound}.
\end{proof}
The previous result establishes the duality gap between \eqref{eqn_problem1} and \eqref{eqn_dual_optimum}. This result brings us a step closer to solve the probabilistic-constrained RL problem. Although, the minimization of the dual function is simple, since the dual function is always convex~{\cite[Chapter 5.2]{boyd2004convex}}, its computation in \eqref{eqn_dual_funtion2} still requires solving \eqref{eqn_agmt_obj}. However, as discussed in Section~\ref{sec_problem_formulation}, the absence of an expression for $\nabla_\theta \mathbb{P} \left(\cap_{t=0}^{T} \{ S_t \in \mathcal{S}_\text{safe}\} |\pi_\theta \right)$ poses a challenge in applying gradient-based methods to do so. Providing this expression as well as corresponding stochastic estimates is the subject of the next section.

%% file: gradient.tex
%!TEX root = root.tex
\section{Safe Policy Gradients}
\label{Gradient of the Probabilistic Constraint}
To solve probabilistic-constrained RL formulations, i.e., problems of the form \eqref{eqn_problem1}, we are required to compute the gradient of the probabilistic constraints with respect to the policy parameter $\theta$. We proceed by defining an important quantity in what follows next. Let $G_t$ be the product of indicator functions $\mathbbm{1}\left(S_u\in\mathcal{S}_{\text{safe}}\right)$ from $u=t$ to $u=T$
\begin{equation}\label{def_G_cumulative_product}
G_t = \prod_{u=t}^T\mathbbm{1}\left(S_u\in\mathcal{S}_{\text{safe}}\right).    
\end{equation}
Having defined this quantity we are now in conditions of providing an expression for the gradient of the probabilistic constraint. This is the subject of the following Theorem.
\begin{theorem}
\label{theorem_safe_policy_gradient}
Let $S_0 \in \mathcal{S}_\text{safe}$, the gradient of the probabilistic safety for a given policy $\pi_\theta$ yields 
\begin{align}\label{eqn_theorem}
   &\nabla_\theta \mathbb{P} \left(\bigcap\limits_{t=0}^{T} \{ S_t \in \mathcal{S}_\text{safe}\} |\pi_\theta, S_0 \right) \nonumber \\
   &=\mathbb{E}\left[ \sum\limits_{t=0}^{T-1}G_1\nabla_{\theta}\log\pi_\theta(A_t\mid S_t)\mid \pi_\theta, S_0\right].
\end{align}
\end{theorem}
\begin{proof}
See Appendix~\ref{appendix_theorem_safe_policy_gradient}.
\end{proof}
This result has been stated in our prior work \cite{chen2023policy}. We provide the proof here for completeness. 
It is worth pointing out that the proof of this result is similar to Policy Gradient Theorems in the literature ~\cite{williams1992simple,sutton1999policy}. To draw parallelisms between them, let us define the cumulative rewards until horizon $T$ starting from $t$%, i.e., the so-called ``reward to-go''~\cite{vitay2020deep} from the transition $(S_t, A_t)$
\begin{align}\label{eqn_return}
  R_t =  \sum\limits_{u=t}^{T} r(S_u, A_u).
\end{align}
Then, the gradient of the value function \eqref{eqn_val_func_finite} can be computed using Policy Gradient Theorem \cite{sutton1999policy} as 
\begin{align}\label{eqn_classical_PG}
   \nabla_\theta V(\theta) =\mathbb{E}\left[\sum\limits_{t=0}^{T-1} R_t \nabla_{\theta}\log\pi_\theta(A_t\mid S_t)\mid \pi_\theta, S_0\right]. 
\end{align}
Despite the similarities between \eqref{eqn_theorem} and \eqref{eqn_classical_PG}  an important difference is evident. In \eqref{eqn_classical_PG} the gradient of the logarithm at time $t$ is multiplied by the cumulative reward from time $t$ until the end of the episode. This means that action $A_t$ is only concerned with the rewards in the episode's future. While in \eqref{eqn_theorem}, this is not the case, and action $A_t$ is concerned with the whole episode. Intuitively, if the agent is unsafe for the first few steps the actions in the future are irrelevant since the episode is considered unsafe as a whole. This will pose additional challenges in estimating the gradient.

Let us start, however, by describing a challenge in the computation of \eqref{eqn_theorem} which is also shared by the computation of $\nabla_\theta V(\theta)$. This is the need to compute expectations with respect to the trajectories of the system. In order to avoid sampling a multitude of trajectories, stochastic approximation methods~\cite{robbins1951stochastic} are often considered. Namely, one can use only one sample trajectory to approximate~\eqref{eqn_classical_PG} 
\begin{align}\label{eqn_classical_PG_estimator}
   \hat{\nabla}_\theta V(\theta) =\sum\limits_{t=0}^{T-1} R_t \nabla_{\theta}\log\pi_\theta(A_t\mid S_t).
\end{align}
% where $R_t$ is the cumulative return from time $t$ until the end of the episode as defined in \eqref{eqn_return}. 
Likewise, applying the stochastic approximation to \eqref{eqn_theorem} yields
\begin{align}\label{eqn_theorem_estimator}
   \hat{\nabla}_\theta \mathbb{P} \! \left(\! \bigcap\limits_{t=0}^{T} \{ S_t \in \mathcal{S}_\text{safe}\} \! \mid \! \pi_\theta, S_0 \! \right) = \sum\limits_{t=0}^{T-1}G_1\nabla_{\theta}\log\pi_\theta(A_t \! \mid \! S_t).
\end{align}
% Notice that by computing the expressions in \eqref{eqn_classical_PG_estimator} and \eqref{eqn_theorem_estimator} over the randomly drawn horizon the estimates derived are unbiased. 
% \santiago{Here we are considering a finite horizon right? Why is the horizon random?} 
%
We term \eqref{eqn_theorem_estimator} ``SPG-REINFORCE''. From \eqref{eqn_theorem_estimator} the additional challenges in estimating the gradient of the probabilistic safety constraint become explicit. Unlike the classic policy gradient \eqref{eqn_classical_PG_estimator} (where the policy parameter is updated at each iteration), under this framework the parameter in \eqref{eqn_theorem_estimator} is only updated when every step of the trajectory is safe, i.e., when $G_1=\prod_{t=1}^T\mathbbm{1}\left(S_t\in\mathcal{S}_{\text{safe}}\right)=1$. To address this issue, one can use a different expression for the gradient which we provide in the next corollary. To do so, we rely on a crucial quantity that captures the product of indicator functions $\mathbbm{1}\left(S_u\in\mathcal{S}_{\text{safe}}\right)$ from $u=0$ to $u=t$
\begin{align}\label{eqn_def_g_t_c}
G_t^c = \prod_{u=0}^t \mathbbm{1}\left(S_u\in\mathcal{S}_{\text{safe}}\right).   
\end{align}
By defining $G_t^c$ we can now provide an alternative expression for the gradient in the following corollary.
\begin{corollary}
Under the hypothesis of Theorem~\ref{theorem_safe_policy_gradient} it holds that
\begin{align}\label{eqn_corollary2}
   &\nabla_\theta \mathbb{P} \left(\bigcap\limits_{t=0}^{T} \{ S_t \in \mathcal{S}_\text{safe}\} |\pi_\theta, S_0 \right) \\
   &=\mathbb{E} \left[ \sum\limits_{t=0}^{T-1} G_t^c \mathbb{E}\left[G_{t+1} \! \mid \! S_t,A_t\right] \nabla_{\theta}\log\pi_\theta(A_t\mid S_t) \! \mid  \! \pi_\theta, S_0 \right].  \nonumber
\end{align}
\end{corollary}
\begin{proof}
Conditioning the right hand side of \eqref{eqn_theorem} with respect to $S_t$ and $A_t$ yields 
\begin{align}\label{eqn_corro_ac}
   &\nabla_\theta \mathbb{P} \left(\bigcap\limits_{t=0}^{T} \{ S_t \in \mathcal{S}_\text{safe}\} |\pi_\theta, S_0 \right)  \\ &=\mathbb{E} \left[\mathbb{E}\left[ \sum\limits_{t=0}^{T-1}G_1\nabla_{\theta}\log\pi_\theta(A_t\mid S_t)\mid S_{t},A_t\right]\mid \pi_\theta, S_0 \right]. \nonumber
\end{align}
Using the fact that $G_1 = G_{t+1} \, G_t^c$ and that $S_0,\ldots,S_t$ are measurable with respect to $S_t,A_t$ yields \eqref{eqn_corollary2}.
\end{proof}
Note that  $\mathbb{E}\left[G_{t+1}\mid S_t,A_t\right]$ is the probability of remaining safe from time $t+1$ until the end of the episode. This information about the future is akin to the Q function in RL problems~\cite[Chapter 6]{sutton2018reinforcement}. Hence, assuming that an estimate of the probability is available, one can run actor-critic~\cite{konda1999actor} type algorithms e.g., DDPG~\cite{lillicrap2015continuous}, TRPO~\cite{schulman2015trust}, PPO~\cite{schulman2017proximal} in this setting as well. In particular, the estimate takes the form 
\begin{align}\label{eqn_corro_ac_estimator}
   &\hat{\nabla}_\theta \mathbb{P} \left(\bigcap\limits_{t=0}^{T} \{ S_t \in \mathcal{S}_\text{safe}\} \mid \pi_\theta, S_0 \right) \\
   &= \sum\limits_{t=0}^{T-1} G_t^c \, \mathbb{E}\left[G_{t+1} | S_t, A_t\right] \nabla_{\theta}\log\pi_\theta(A_t\mid S_t).  \nonumber
\end{align}
We term \eqref{eqn_corro_ac_estimator} ``SPG-Actor-Critic''. An advantage of ``SPG-Actor-Critic'' is that if the episode is safe until some time $0<t<T$, we have $G_u^c =1$, with $u\leq t$. Hence the estimate of the gradient is not zero and the policy will be updated. This is in contrast to the expression in \eqref{eqn_theorem_estimator}. This fact can also be interpreted as an improved estimate with reduced variance as in actor-critic methods. The next theorem is a step forward in confirming this intuition.
%
%
%
% \begin{theorem}\label{theorem_lower_variance}
% Let $X=\sum_{t=0}^{T-1} G_1$ and $Y=\sum_{t=0}^{T-1} G_t^c \, \mathbb{E}\left[G_{t+1} | S_t, A_t\right]$, where $G_1$ and $G_{t+1}$ are defined by \eqref{def_G_cumulative_product}, and $ G_t^c$ follows the definition in \eqref{eqn_def_g_t_c}. It holds that $\var{X} > \var{Y}$.
% \end{theorem}
\begin{theorem}\label{theorem_lower_variance}
Consider $G_1, \, G_{t+1}$ as defined in \eqref{def_G_cumulative_product} and $ G_t^c$ in \eqref{eqn_def_g_t_c}. Let $X=\sum_{t=0}^{T-1} G_1$ and $Y=\sum_{t=0}^{T-1} G_t^c \, \mathbb{E}\left[G_{t+1} | S_t, A_t\right]$. It holds that $\var{X} > \var{Y}$.
\end{theorem}
\begin{proof}
Denote by $\mathcal{T}$ the set $\{0, 1, \cdots, T\}$.
We proceed by stating the fact that $G_t^c$ is measurable with respect to $S_t, A_t$. Then, by virtue of the towering property~\cite[Theorem 4.1.13]{durrett2019probability} yields
$\mathbb{E}\left[ G_1 \right] = \mathbb{E}\left[ G_t^c \, \mathbb{E}\left[G_{t+1} | S_t, A_t\right]  \right]$, $\forall t\in\mathcal{T}$. Therefore, by exchanging the sum and the expectation we obtain
\begin{align}\label{eqn_same_expectation}
    \mathbb{E}\left[ X \right] = \mathbb{E}\left[ Y \right].
\end{align}
Since $\var{X} = \mathbb{E}\left[ X^2 \right] - \mathbb{E}\left[ X \right]^2$ and $\var{Y} = \mathbb{E}\left[ Y^2 \right] - \mathbb{E}\left[ Y \right]^2$ it follows that 
\begin{align}\label{eqn_variance_diff0}
    \var{X} - \var{Y} = \mathbb{E}\left[ X^2 \right] - \mathbb{E}\left[ Y^2 \right].
\end{align}
By definitions of $X$ and $Y$ and expanding the square we obtain
% \begin{align}\label{eqn_variance_diff}
%     &var(X) - var(Y) \nonumber \\
%     &= \mathbb{E}\left[ \left(\sum_{t=0}^{T-1} G_1 \right)^2 \right] - \mathbb{E}\left[ \left(\sum_{t=0}^{T-1} G_t^c \, \mathbb{E}\left[G_{t+1} | S_t, A_t\right] \right)^2 \right] \nonumber \\
%     &=\mathbb{E}\left[ \sum_{t, u \in \mathcal{T}} G_1 ^2 \right] \nonumber \\
%     &-\mathbb{E}\left[ \sum_{t, u \in \mathcal{T}} G_t^c \, \mathbb{E}\left[G_{t+1} | S_t, A_t\right] G_u^c \, \mathbb{E}\left[G_{u+1} | S_u, A_u\right] \right].
% \end{align}
%
\begin{align}\label{eqn_variance_diff}
    &\mathbb{E}\left[ X^2 \right] = \mathbb{E}\left[ \sum_{t, u \in \mathcal{T}} G_1 ^2 \right],  \\
    &\mathbb{E}\left[ Y^2 \right] = \mathbb{E}\left[ \sum_{t, u \in \mathcal{T}} G_t^c \, \mathbb{E}\left[G_{t+1} | S_t, A_t\right] G_u^c \, \mathbb{E}\left[G_{u+1} | S_u, A_u\right] \right] \nonumber.
\end{align}
We claim that $\forall {t, u \in \mathcal{T}}$ the following inequality holds
\begin{align}\label{eqn_G1_Gt_Gu}
    \mathbb{E}\left[ G_1 ^2  -  G_t^c \, \mathbb{E}\left[G_{t+1} | S_t, A_t\right] G_u^c \, \mathbb{E}\left[G_{u+1} | S_u, A_u\right] \right] \geq 0.
\end{align}
Using the linearity of the expectation along with the previous inequality and the expressions \eqref{eqn_variance_diff0}--\eqref{eqn_variance_diff} completes the proof of the result.

We are left to prove \eqref{eqn_G1_Gt_Gu}. Denote by $D_{XY}$ the left-hand side of \eqref{eqn_G1_Gt_Gu}. We divide the analysis into two cases: $t=u$ and $t \not= u$. For $t=u$, rewrite $D_{XY}$ as 
    \begin{align}
        D_{XY} = \mathbb{E}\left[ G_1^2  -  \left(G_t^c \, \mathbb{E}\left[G_{t+1} | S_t, A_t\right] \right)^2 \right].
    \end{align}
    Since $G_t^c$ is measurable with respect to $S_t, A_t$, the previous equation is equivalent to 
    \begin{align}\label{eqn_t_yes_u}
        D_{XY} &= \mathbb{E}\left[ G_1^2  -  \left(\mathbb{E}\left[G_t^c \,  G_{t+1} | S_t, A_t\right] \right)^2 \right] \nonumber \\
        &= \mathbb{E}\left[ G_1^2  -  \left(\mathbb{E}\left[G_1 | S_t, A_t\right] \right)^2 \right],
    \end{align}
where the last equation directly follows from the definition of $G_1$. Taking the conditional expectation on $S_t, A_t$ for the right-hand side of \eqref{eqn_t_yes_u} yields
    \begin{align}\label{eqn_conditional_var}
        D_{XY} = \mathbb{E}\left[ \mathbb{E}\left[ G_1^2 | S_t, A_t \right]  -  \left(\mathbb{E}\left[G_1 | S_t, A_t\right] \right)^2 \right].
    \end{align}
By the definition of conditional variance, \eqref{eqn_conditional_var} reduces to
    \begin{align}
        D_{XY} = \mathbb{E}\left[ \var{G_1 | S_t, A_t} \right] \geq 0.
    \end{align}

For $t \not= u$, let us consider $t < u$ (the case of $t > u$ is analogous). Using the fact that $G_1=G_t^cG_{t+1}=G_u^cG_{u+1}$, and $G_t^c$ and $G_u^c$ are in the past of $t$ and $u$, one can rewrite $D_{XY}$ in \eqref{eqn_G1_Gt_Gu} as
\begin{align}\label{eqn_t_u_back2}
D_{XY} = \mathbb{E}\left[G_1^2 -\mathbb{E}\left[G_1| S_t,A_t\right]\mathbb{E}\left[G_1| S_u,A_u\right]\right].
\end{align}
%
% \red{Since $G_t^c$, $G_u^c$ are measurable with respect to $S_u, A_u$, taking the conditional expectation with respect to $S_u, A_u$ in the right-hand side of the previous equation yields
%     \begin{align}
%         D_{XY} &= \mathbb{E}[ G_t^c G_u^c \mathbb{E}\left[G_{t+1} G_{u+1} | S_u, A_u \right] \nonumber \\
%         &- G_t^c G_u^c \mathbb{E}\left[G_{t+1} | S_t, A_t\right] \mathbb{E}\left[G_{u+1} | S_u, A_u\right] ],
%     \end{align}
% where $\mathbb{E}\left[G_{t+1} | S_t, A_t\right]$ maintains to be conditioned on $S_t, A_t$ due to the fact that $t < u$. Simplify the equation above by rearranging $G_t^c G_u^c$ into the corresponding conditional expectation
%     \begin{align}\label{eqn_t_u_back}
%         D_{XY} &= \mathbb{E}[ \mathbb{E}\left[G_{t+1}G_t^c G_{u+1}G_u^c | S_u, A_u \right] \nonumber \\
%         &- \mathbb{E}\left[G_{t+1}G_t^c | S_t, A_t\right] \mathbb{E}\left[G_{u+1}G_u^c | S_u, A_u\right] ].
%     \end{align}
%
%
% Applying the definition of $G_1$ to \eqref{eqn_t_u_back} yields
%     \begin{align}\label{eqn_t_u_back2}
%         D_{XY} = \mathbb{E}[ \mathbb{E}\left[G_1^2 | S_u, A_u \right] - \mathbb{E}\left[G_1 | S_t, A_t\right] \mathbb{E}\left[G_1 | S_u, A_u\right] ].
%     \end{align}
% }
%
Taking the conditional expectation on $S_t, A_t$ for the right-hand side of \eqref{eqn_t_u_back2} using the fact that $t<u$ and the towering property~\cite[Theorem 4.1.13]{durrett2019probability} yields
    \begin{align}\label{eqn_t_u_back3}
        D_{XY} &= \mathbb{E}[ \mathbb{E}\left[G_1^2 | S_t, A_t \right] - \mathbb{E}\left[G_1 | S_t, A_t\right]^2 ] \nonumber \\
        &= \mathbb{E}\left[ \var{G_1 | S_t, A_t}\right] \geq 0.
    \end{align}
This completes the proof of the result.
\end{proof}

The intuition of the modified estimate in \eqref{eqn_corro_ac_estimator} is that it leads to an estimate with smaller variance as compared to \eqref{eqn_theorem_estimator}. Note that this claim is merely an intuition, as it does not incorporate the term $\nabla_{\theta}\log\pi_\theta(A_t|S_t)$. Equipped with expressions to estimate the gradient of the safety probability, in the next section we turn our attention to algorithms aiming to solve the relaxations \eqref{eqn_agmt_obj} and \eqref{eqn_dual_optimum} of problem \eqref{eqn_problem1}.

%% file: algorithms.tex
%!TEX root = root.tex

\section{Learning Safe Policies}
\label{Losses_of_Imposing_Relaxation}
It is worth highlighting that both the regularization method~\cite{censor1977pareto} and primal-dual method~\cite{paternain2022safe} straightforwardly apply in the cumulative-constrained problem like \eqref{eqn_problem2_mirror}. To do so, one can construct a new reward that incorporates the constraint
\begin{equation}\label{eqn_reward_risk}
   r_\mu(S_t, A_t)=r(S_t, A_t) + \mu \mathbbm{1} ( S_t \in \mathcal{S}_\text{safe}),
\end{equation}
where $\mu$ is a parameter that trades-off the reward for safety (analogous to $\lambda$ in \eqref{eqn_agmt_obj}).
Having defined the new reward function, off-the-shelf RL algorithms, e.g.,~\cite{williams1992simple, sutton1999policy, lillicrap2015continuous, schulman2015trust, schulman2017proximal, haarnoja2018soft} are available to solve \eqref{eqn_problem2_mirror}. In particular, we consider a stochastic approximation of Policy Gradient Theorem~\cite{sutton1999policy} which yields the policy gradient $\hat{\nabla} V_\mu (\theta) = \sum_{t=0}^{T-1} r_{\mu}(S_t, A_t) \nabla_{\theta}\log\pi_\theta(A_t| S_t)$.

We then turn our attention to solve \eqref{eqn_agmt_obj}. To do so, we incorporate $\hat{\nabla}_\theta \mathbb{P} \! \left(\! \cap_{t=0}^{T} \{ S_t \in \mathcal{S}_\text{safe}\} \! \mid \! \pi_\theta, S_0 \! \right)$ into \eqref{eqn_classical_PG_estimator} with a weight $\lambda$
\begin{align}\label{eqn_spg}
   \hat{\nabla}_\theta\ccalL(\theta,\lambda)  =  \hat{\nabla}_\theta V(\theta)  +  \lambda \hat{\nabla}_\theta \mathbb{P} \! \left(\bigcap\limits_{t=0}^{T} \{ S_t \in \mathcal{S}_\text{safe}\} \! \!\mid \! \! \pi_\theta, S_0 \right),
\end{align}
where $\hat{\nabla}_\theta \mathbb{P} \! \left(\! \cap_{t=0}^{T} \{ S_t \in \mathcal{S}_\text{safe}\} \! \mid \! \pi_\theta, S_0 \! \right)$ can be either ``SPG-REINFORCE'' as in \eqref{eqn_theorem_estimator} or ``SPG-Actor-Critic'' as in \eqref{eqn_corro_ac_estimator}.

The regularization method for solving both constrained problems yields the following stochastic gradient ascent update
\begin{align}\label{update_rule_probabilistic}
    \theta^{k+1} &=\theta^k + \eta_\theta \, \text{GRADIENT},
\end{align}
where $\eta_\theta$ denotes the stepsize, and $\text{GRADIENT}$ follows from $\hat{\nabla} V_\mu (\theta)$ and \eqref{eqn_spg} in the cumulative-constrained and probabilistic-constrained settings, respectively.

Alternatively, one can still employ the primal-dual method in the probabilistic-constrained setting that can iteratively update $\lambda$ while updating $\theta$ in the same way as in \eqref{update_rule_probabilistic}. Therefore, we propose in the remainder of this section a Safe Primal-Dual algorithm summarized under Algorithm~\ref{alg_pd}, accompanied by analyses in terms of the convergence, near-optimality, and feasibility. Algorithm~\ref{alg_pd} in which SPGs play a crucial role is specifically designed for probabilistic-constrained RL (our focus), albeit it exhibits an analogous structure to the primal-dual framework~\cite{paternain2022safe} employed in the cumulative-constrained setting. Note that in Step 4 of Algorithm~\ref{alg_pd} we present two versions of the primal update. The first one, referred to as Ideal, is primarily intended for subsequent theoretical analyses. On the other hand, the second variant, known as Practical, is the version employed in our numerical experiments (Section~\ref{Numerical_Experiments}). Regarding updating the dual variable, let us define $g(\theta)$ as
\begin{align}\label{eqn_def_g}
    g(\theta) = \mathbb{P} \left(\bigcap\limits_{t=0}^{T} \{ S_t \in \mathcal{S}_\text{safe}\} | \pi_\theta \right)  -  (1-\delta).
\end{align}
Since $\mathbb{P} \left(\cap_{t=0}^{T} \{ S_t \in \mathcal{S}_\text{safe}\} | \pi_\theta \right)$ can only be estimated, Step 5 considers an unbiased estimate of \eqref{eqn_def_g} in the update, i.e.,
\begin{align}\label{eqn_def_g_hat}
    \hat{g}(\theta) = \prod_{t=0}^T\mathbbm{1}\left(S_t\in\mathcal{S}_{\text{safe}} | \pi_\theta \right) - (1-\delta).
\end{align}
%%%%%%%%%%%%%%%%%%
%%%%%%%%%%%%%%%%%%
%%%%%%%%%%%%%%%%%%
%%%%%%%%%%%%%%%%%%
%%%%%%%%%%%%%%%%%%
%%%%%%%%%%%%%%%%%%
\begin{algorithm}[t]
  \caption{Safe Primal-Dual}
  \label{alg_pd} 
\begin{algorithmic}[1]
 \renewcommand{\algorithmicrequire}{\textbf{Input:}}
 \renewcommand{\algorithmicensure}{\textbf{Output:}}
 \Require $\theta^0,\lambda^0,T, \eta_\theta,\eta_\lambda, \delta, \epsilon$
  \For {$k = 0,1,\ldots $}
  \State Simulate a trajectory with the policy $\pi_{\theta^k}(a | s)$
  \State Estimate $\hat{\nabla}_\theta\ccalL(\theta^k,\lambda^k)$ as in \eqref{eqn_spg}
  \State(Ideal) Keep updating the primal variable
  $$
  \theta^{k+1} = \theta^k+ \eta_{\theta} \hat{\nabla}_\theta\ccalL(\theta^k,\lambda^k)
   $$
  \Statex \quad \, until $\ccalL(\theta^{k+1},\lambda^k) \geq \ccalL(\theta^*,\lambda^k) - \epsilon$
   \Statex \quad \, (Practical) Update the primal variable once
%   $$
%   \theta^{k+1} = \theta^k+ \eta_{\theta} \hat{\nabla}_\theta\ccalL(\theta^k,\lambda^k)
% $$
   \State Update the dual variable using \eqref{eqn_def_g_hat}
  $$
  \lambda^{k+1} = \left[\lambda^k- \eta_{\lambda}\hat{g}(\theta^{k+1}) \right]_+
    $$
  \EndFor
% \Return $\theta$, $s_{T_Q}$ 
 \end{algorithmic}
 \end{algorithm}
%%%%%%%%%%%%%%%%%%
%%%%%%%%%%%%%%%%%%
%%%%%%%%%%%%%%%%%%
%%%%%%%%%%%%%%%%%%

% \red{But more importantly, strong duality
% does not imply that the set of primal optimal theta and
% dual optimal theta coincides nor that all dual optimal
% theta's are primal feasible.}

Furthermore, we assert that by employing Algorithm~\ref{alg_pd}, the dual value is guaranteed to converge to a small neighborhood surrounding the dual optimum $D^\star$. This claim is formally established and proved in the following theorem.

%%%%%%%%%%%%%%%%%%
%%%%%%%%%%%%%%%%%%
%%%%%%%%%%%%%%%%%%
%%%%%%%%%%%%%%%%%%
\begin{theorem}
\label{proposition_dual_converge}
Let $\lambda^\star \in \argmin_{\lambda \in \mathbb{R}_+} d(\lambda) $, i.e., $d(\lambda^\star) = D^\star$, where $d(\lambda)$ is the dual function defined in \eqref{eqn_dual_funtion2}. Consider the sequence $\{\lambda^u\}_{u=0}^k$ generated by Algorithm~\ref{alg_pd} with dual stepsize $\eta_\lambda$ and safety level $\delta < 1/2$. Assume that the primal variable $\theta$ is updated by Ideal version of Step 4 with accuracy $\epsilon>0$. Define $\lambda^{\text{best}}_k = \argmin_{\lambda^u \in \{\lambda^u\}_{u=0}^k }  d(\lambda)$, it holds that
\begin{align}\label{eqn_theorem_dual_converge}
    0 \leq \mathbb{E} \left[ d(\lambda^{\text{best}}_k)  \right] - D^\star \leq \frac{\mathbb{E} \left[ (\lambda^0 - \lambda^\star)^2 \right]}{2 \eta_\lambda k} + \frac{\eta_\lambda (1-\delta)^2}{2} + \epsilon.
\end{align}
\end{theorem}
\begin{proof}
Before proceeding, it is worth pointing out that Theorem~\ref{proposition_dual_converge} holds for both single-constraint and multi-constraint problems in an analogous manner. For simplicity and without losing generality, we provide here the proof for the single-constraint problem.

The leftmost inequality \eqref{eqn_theorem_dual_converge} directly follows from the definition of $\lambda^\star$. Thus, we focus on the proof of the rightmost inequality of \eqref{eqn_theorem_dual_converge}. Recall the dual update in Algorithm~\ref{alg_pd}
\begin{align}\label{eqn_dual_update}
    \lambda^{k+1} = \left[\lambda^k- \eta_{\lambda}\hat{g}(\theta^{k+1})\right]_+.
\end{align}
Then, we obtain
\begin{align}
     (\lambda^{k+1} - \lambda^\star )^2 &= ( \left[\lambda^k- \eta_{\lambda} \hat{g}(\theta^{k+1}) \right]_+ - \lambda^\star )^2 \nonumber \\
    &\leq ( \lambda^k- \lambda^\star -\eta_{\lambda} \hat{g}(\theta^{k+1}) )^2,
\end{align}
where the last inequality follows from the non-expansiveness of the projection. Expanding the norm square and upper bounding $\hat{g}(\theta^{k+1})^2$ by $(1-\delta)^2$ yields
% \begin{align}
%     &( \lambda^{k+1} - \lambda^\star )^2  \\ 
%     &\leq (\lambda^k- \lambda^\star )^2 +  ( \eta_{\lambda} \hat{g}(\theta^{k+1}) )^2 -2 \eta_{\lambda} \hat{g}(\theta^{k+1}) (\lambda^{k} - \lambda^\star). \nonumber
% \end{align}
%
% Further upper bounding $\hat{g}(\theta^{k+1})^2$ by $(1-\delta)^2$ yields
\begin{align}
    &( \lambda^{k+1} - \lambda^\star )^2 \nonumber \\ &\leq (\lambda^k- \lambda^\star )^2 +   \eta_{\lambda}^2 (1-\delta)^2 -2 \eta_{\lambda} \hat{g}(\theta^{k+1}) (\lambda^{k} - \lambda^\star).
\end{align}
Taking the conditional expectation on $\theta^{k+1}$ for both sides of the previous inequality yields
\begin{align}
    &\mathbbm{E} \left[ ( \lambda^{k+1} - \lambda^\star )^2 | \theta^{k+1} \right]  \\ 
    &\! \leq \! \mathbbm{E} \left[ (\lambda^k- \lambda^\star )^2 \! + \! \eta_{\lambda}^2 (1-\delta)^2 \!-\! 2 \eta_{\lambda} \hat{g}(\theta^{k+1}) (\lambda^{k} - \lambda^\star) | \theta^{k+1} \right]  \nonumber \\
    & \! = \! (\lambda^k- \lambda^\star )^2 \!+\!   \eta_{\lambda}^2 (1-\delta)^2 \!-\!2 \eta_{\lambda} (\lambda^{k} - \lambda^\star) \mathbbm{E} \left[ \hat{g}(\theta^{k+1}) | \theta^{k+1} \right], \nonumber
\end{align}
where the last equation holds by the fact that $\lambda^k$ is deterministic given $\theta^{k+1}$. By definition of $g(\cdot)$ in \eqref{eqn_def_g} we further have
\begin{align}\label{eqn_before_lemma1}
    &\mathbbm{E} \left[ ( \lambda^{k+1} - \lambda^\star )^2 | \theta^{k+1} \right] \nonumber \\ 
    &\leq (\lambda^k- \lambda^\star )^2 +   \eta_{\lambda}^2 (1-\delta)^2  -2 \eta_{\lambda} (\lambda^{k} - \lambda^\star)  g(\theta^{k+1}).
\end{align}
We proceed to apply Lemma~\ref{lemma_primal_error} with $\lambda^\star$ and $\lambda^{k}$ based on $\theta^{k+1}$
\begin{align}\label{eqn_primal_err_again}
    D^\star = d(\lambda^\star) \geq d(\lambda^{k}) + (\lambda^\star - \lambda^{k})  g(\theta^{k+1}) - \epsilon.
\end{align}
The previous inequality is equivalent to
\begin{align}\label{eqn_primal_err_3}
    (\lambda^\star - \lambda^{k})  g(\theta^{k+1}) \leq D^\star - d(\lambda^{k}) + \epsilon.
\end{align}
Substituting the previous expression into \eqref{eqn_before_lemma1} yields
\begin{align}
    &\mathbbm{E} \left[ ( \lambda^{k+1} - \lambda^\star )^2 | \theta^{k+1} \right] \nonumber \\ 
    &\leq (\lambda^k- \lambda^\star )^2 +   \eta_{\lambda}^2 (1-\delta)^2 + 2 \eta_{\lambda} \left( D^\star - d(\lambda^{k}) + \epsilon \right).
\end{align}
Taking the expectation again on both sides of previous inequality and by virtue of towering property yields
\begin{align}
    \mathbbm{E} \left[ ( \lambda^{k+1} - \lambda^\star )^2 \right] &\leq \mathbbm{E} \left[ (\lambda^k- \lambda^\star )^2  \right] +  \\
    &  \eta_{\lambda}^2 (1-\delta)^2 + 2 \eta_{\lambda} \left( D^\star - \mathbbm{E} \left[ d(\lambda^{k}) \right] + \epsilon \right), \nonumber
\end{align}
where $\lambda^\star, \eta_\lambda, \delta, \epsilon$ are deterministic. Simply iterating backwards the previous inequality directly yields
\begin{align}
    \mathbbm{E} \left[ ( \lambda^{k} - \lambda^\star )^2 \right] &\leq \mathbbm{E} \left[ (\lambda^{k-1}- \lambda^\star )^2  \right] +   \\
    &\eta_{\lambda}^2 (1-\delta)^2 + 2 \eta_{\lambda} \left( D^\star - \mathbbm{E} \left[ d(\lambda^{k-1}) \right] + \epsilon \right). \nonumber
\end{align}
Keeping unrolling such series of inequalities and summing them up derives
\begin{align}
    \mathbbm{E} \left[ ( \lambda^{k} - \lambda^\star )^2 \right] \leq &\mathbbm{E} \left[ (\lambda^{0}- \lambda^\star )^2  \right] +  k \eta_{\lambda}^2 (1-\delta)^2 \nonumber \\
    &+ 2 \eta_{\lambda} (k D^\star  + k \epsilon  - \sum_{u=0}^{k-1} \mathbbm{E} \left[ d(\lambda^{u}) \right] ).
\end{align}
Notice that the right-hand side of the previous inequality is non-negative. Then, rearranging and dividing by $2 \eta_{\lambda} k$ in both sides yields
\begin{align}\label{eqn_convexity1}
    &\frac{1}{k} \sum_{u=0}^{k-1} \mathbbm{E} \left[ d(\lambda^{u}) \right] - D^\star
    \leq  \frac{\mathbbm{E} \left[ (\lambda^{0}- \lambda^\star )^2  \right]}{2 \eta_{\lambda} k}  + \frac{\eta_{\lambda} (1-\delta)^2}{2} +\epsilon.
\end{align}
% The convexity of dual function $d(\cdot)$ clarifies
% \begin{align}
%     \mathbbm{E} \left[ d(\frac{1}{k} \sum_{u=0}^{k-1}  \lambda^{u}) \right] \leq \frac{1}{k} \sum_{u=0}^{k-1} \mathbbm{E} \left[ d(\lambda^{u}) \right].
% \end{align}
% Then \eqref{eqn_convexity1} reduces to
% \begin{align}
%     &\mathbbm{E} \left[ d(\frac{1}{k} \sum_{u=0}^{k-1}  \lambda^{u}) \right] - d(\lambda^\star)
%     &\leq  \frac{\mathbbm{E} \left[ (\lambda^{0}- \lambda^\star )^2  \right]}{2 \eta_{\lambda} k}  + \frac{\eta_{\lambda} (1-\delta)^2}{2} +\epsilon.
% \end{align}
The definition of $\lambda^{\text{best}}_k$ completes the proof of Theorem~\ref{proposition_dual_converge}.

\end{proof}
%%%%%%%%%%%%%%%%%%
%%%%%%%%%%%%%%%%%%
%%%%%%%%%%%%%%%%%%
%%%%%%%%%%%%%%%%%%

Having demonstrated the convergence of $\lambda^{\text{best}}_k$ that is obtained from the sequence of dual variables generated by Algorithm~\ref{alg_pd}, we can further establish that the primal update also guarantees proximity to the primal optimum $P^\star$. The following theorem addresses this aspect explicitly.

%%%%%%%%%%%%%%%%%%
%%%%%%%%%%%%%%%%%%
%%%%%%%%%%%%%%%%%%
%%%%%%%%%%%%%%%%%%
\begin{theorem}\label{theorem_primal_bound}
Consider $P^\star$ and $\delta$ as in \eqref{eqn_problem1}, and let hypotheses of Theorem~\ref{proposition_dual_converge} hold. It holds that
\begin{align}\label{eqn_v_primal_lam_best}
    \liminf_{k \to \infty} \frac{1}{k+1} \sum_{u=0}^k \mathbbm{E} \left[V(\theta^{u+1}) \right] \geq P^\star - \frac{\eta_\lambda (1-\delta)^2}{2} - \epsilon.
\end{align}
\end{theorem}
\begin{proof}
For $\forall k > 0$, the following chain of inequalities holds due to the fact that the dual function is an upper bound on the value of the primal and it is convex~{\cite[Chapter 5.2]{boyd2004convex}}
\begin{align}
    P^\star \leq d\left( \frac{1}{k+1} \sum_{u=0}^k \lambda^u\right) \leq \frac{1}{k+1} \sum_{u=0}^k  d(\lambda^u).
\end{align}

Since the primal variable $\theta^{u+1}$ is updated by Ideal version of Algorithm~\ref{alg_pd} with accuracy $\epsilon$, i.e., $d(\lambda^u) \leq V(\theta^{u+1}) + \lambda^u \hat{g}(\theta^{u+1}) + \epsilon$, we obtain
\begin{align}
    P^\star \leq \frac{1}{k+1} \sum_{u=0}^k \left( V(\theta^{u+1}) + \lambda^u \hat{g}(\theta^{u+1}) \right) + \epsilon. 
\end{align}
Reordering terms and taking expectation in both sides yields
\begin{align}\label{eqn_state_augment1}
    \frac{1}{k+1} \sum_{u=0}^k \mathbbm{E} [V(\theta^{u+1}) ] \geq  P^\star \! - \! \frac{1}{k+1} \sum_{u=0}^k \mathbbm{E} [\lambda^u \hat{g}(\theta^{u+1})] - \epsilon.
\end{align}

By virtue of Lemma 1 in
\cite{calvo2021state} and the fact that $\hat{g}(\theta^{k+1})^2 \leq (1-\delta)^2$, we obtain
\begin{align}\label{eqn_state_augment2}
    \frac{1}{k+1} \sum_{u=0}^k \mathbbm{E} [ \lambda^u \hat{g}(\theta^{{u+1}}(\lambda^u)) ]  \leq \frac{\eta_\lambda (1-\delta)^2}{2} + \frac{(\lambda^0)^2}{2 \eta_\lambda (k+1)}.
\end{align}
Combining \eqref{eqn_state_augment1} and \eqref{eqn_state_augment2} yields
\begin{align}
    \frac{1}{k+1} \sum_{u=0}^k \mathbbm{E} [V(\theta^{u+1}) ] \geq & P^\star \!-\! \frac{\eta_\lambda (1-\delta)^2}{2} \!-\! \frac{(\lambda^0)^2}{2 \eta_\lambda (k+1)} \!-\! \epsilon.
\end{align}
The proof is completed by taking the limit inferior in both sides of the previous inequality.
\end{proof}
%%%%%%%%%%%%%%%%%%
%%%%%%%%%%%%%%%%%%
%%%%%%%%%%%%%%%%%%
%%%%%%%%%%%%%%%%%%

Theorem~\ref{theorem_primal_bound} indicates that the limit inferior of the average of the sequence generated by Algorithm~\ref{alg_pd} approximates well the value of $P^\star$. On the other hand, in principle, the limit superior has the potential to be significantly better than $P^\star$, which would lead to a safety violation. In the following result, we substantiate that this is not the case, i.e., the sequence derived by Algorithm~\ref{alg_pd} is feasible on average.

%%%%%%%%%%%%%%%%%%
%%%%%%%%%%%%%%%%%%
%%%%%%%%%%%%%%%%%%
%%%%%%%%%%%%%%%%%%
\begin{theorem}
\label{theorem_feasibility}
Let hypotheses of Theorem~\ref{proposition_dual_converge} hold. It holds that
% \red{\begin{align}\label{eqn_first_case}
%     \liminf_{k \to \infty}  \frac{1}{k+1} \sum_{u=0}^k g(\theta^{u+1} )   \geq 0.
% \end{align}}
%
\begin{align}\label{eqn_first_case}
    \liminf_{k \to \infty}  \frac{1}{k+1} \sum_{u=0}^k  \mathbbm{E}_{\theta^{u}} \! \left[ \mathbb{P} \! \left(\bigcap\limits_{t=0}^{T} \{ S_t \in \mathcal{S}_\text{safe}\} | \pi_{\theta^{u}}\right) \right] \! \geq  1-\delta.
\end{align}
\end{theorem}
\begin{proof}
Consider $\hat{g}(\cdot)$ as defined in \eqref{eqn_def_g_hat}. By virtue of the non-expansiveness
of the projection, \eqref{eqn_dual_update} can be re-written as
\begin{align}
    \lambda^{k+1} \geq \lambda^k- \eta_{\lambda}\hat{g}(\theta^{k+1}).
\end{align}
Using the previous upper bound recursively yields
\begin{align}
    \lambda^{k+1} \geq \lambda^0- \eta_{\lambda}\sum_{u=0}^{k} \hat{g}(\theta^{u+1}).
\end{align}
Rearranging the previous inequality and dividing by $\eta_\lambda (k+1)$ on both sides results in
\begin{align}
     \frac{1}{k+1} \sum_{u=0}^{k} \hat{g}(\theta^{u+1}) \geq  \frac{\lambda^0 - \lambda^{k+1}}{\eta_{\lambda} (k+1)}.
\end{align}
As described in \eqref{eqn_def_g} and \eqref{eqn_def_g_hat}, $\hat{g}(\cdot)$ is an unbiased estimate of $g(\cdot)$. Subsequently, taking the expectation on the whole sequence in both sides of the previous inequality yields
\begin{align}\label{eqn_result_theo6}
     \frac{1}{k+1} \sum_{u=0}^k \mathbbm{E}_{\theta^{u+1}} \left[g(\theta^{u+1}) \right] \geq  \frac{ \mathbbm{E} \left[ \lambda^0 -  \lambda^{k+1} \right]}{\eta_{\lambda} (k+1)}.
\end{align}
Since $\lambda^0, \eta_\lambda$ are bounded constant, by virtue of Lemma~\ref{lemma_lam_bounded} in Appendix~\ref{appendix_lemma_primal_error} the limit inferior of the right-hand side of \eqref{eqn_result_theo6} is zero.
Then, the definition of $g(\cdot)$ as shown in \eqref{eqn_def_g} completes the proof of the result.
\end{proof}
%%%%%%%%%%%%%%%%%%
%%%%%%%%%%%%%%%%%%
%%%%%%%%%%%%%%%%%%
%%%%%%%%%%%%%%%%%%
It is important to acknowledge that Theorem~\ref{theorem_feasibility} substantiates on average the expected feasibility of the sequence generated by Algorithm~\ref{alg_pd}, even the sequence possesses the potential to approach the optimal value $P^\star$. Having proposed the Safe Primal-Dual (Algorithm~\ref{alg_pd}), and conducted the analyses regarding the convergence, near-optimality, and feasibility. In the subsequent section, we will evidently evaluate Algorithm~\ref{alg_pd} as well as the regularization method, and validate the theoretical assertions made in the preceding sections.

%% file: numerical_results.tex
\section{Numerical Results}
\label{Numerical_Experiments}
In this section, we demonstrate the numerical performance of the methods presented in Section~\ref{Losses_of_Imposing_Relaxation}. To do so we consider a navigation task in an environment with hazardous obstacles (see Figure~\ref{configuration}), the lunar lander problem in OpenAI Gym~\cite{brockman2016openai} characterized by complex dynamics (see Figure~\ref{lunar_lander}), and the safe RL benchmark simulator Safety Gym~\cite{ray2019benchmarking}.
% \red{
% In Section~\ref{Losses_of_Imposing_Relaxation}, we have proposed two SPGs as in \eqref{eqn_spg}, the corresponding regularization method as in \eqref{update_rule_cumulative} and \eqref{update_rule_probabilistic}, and the primal-dual method as in Algorithm~\ref{alg_pd}. In this section, to study the behavior of these proposed gradients and methods as well as to validate the claims in Theorem~\ref{theorem_P_star_Ptilder_star} and Theorem~\ref{theorem_lower_variance}, we consider a navigation task in an environment filled with hazardous obstacles (See Figure~\ref{configuration}) and a lunar lander problem in OpenAI Gym framework~\cite{brockman2016openai} characterized by complex dynamics (See Figure~\ref{lunar_lander}).}
%
\subsection{Navigation in a cluttered environment}
To validate the effectiveness of Algorithm~\ref{alg_pd} (Safe Primal-Dual) as well as the claims in Theorem~\ref{theorem_P_star_Ptilder_star} and Theorem~\ref{theorem_lower_variance}, we consider the problem of navigating an environment cluttered with obstacles as illustrated in Figure~\ref{configuration}. The state of the agent is the continuous $\mathcal{S}=[0, 10] \times [0, 10]$ representing the horizontal ($x$) and vertical ($y$) positions.

\begin{figure}
\centering
\includegraphics[width=0.6\columnwidth]{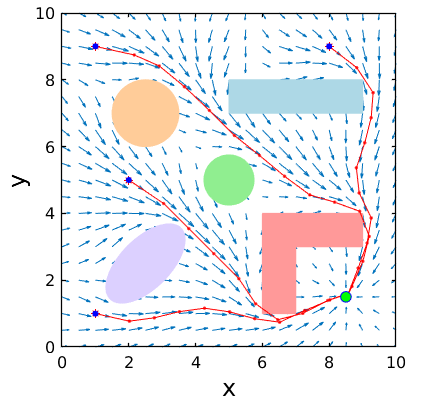}
\caption{Navigation policy learned after 250,000 episodes of training employing Algorithm~\ref{alg_pd} with SPG-REINFORCE and the practical version in Step 4 where $\eta_\theta = 0.02$, $\eta_\lambda = 0.002$, $1-\delta=0.95$. The agent is trained to navigate starting from (1, 1), (1, 9), (2, 5), (8, 9) to the goal (8.5, 1.5).}
\label{configuration}
\end{figure} 
The agent's action $a$ is a two-dimensional velocity. For a given state and action at time $t$, the state evolves according to $s_{t+1}=s_t+ a_t T_s$ with $T_s=0.05$. The policy of the agent is a multivariate Gaussian distribution
\begin{align}\label{gaussian_policy}
    \pi_\theta (a | s)\! =\! \frac{1}{\sqrt{(2\pi)^2 |\Sigma|}} \exp \! \left(\!\! -\frac{1}{2} \left(a-\phi_\theta(s)\right)^\top \!\! \Sigma^{-1} \!\! \left(a-\phi_\theta(s)\right) \! \! \right),
\end{align}
where $\phi_\theta(s)$ and $\Sigma$ denote the mean and covariance matrix of the Gaussian policy. We parameterize $\phi_\theta(s)$ as a linear combination of Radial Basis Functions (RBFs)
\begin{align}\label{RBF}
    \phi_\theta (s) = \sum_{k=1}^{d} \theta_k \exp \left(-\frac{||s-\Bar{s}_k||^2}{2\sigma^2} \right),
\end{align}
where $\theta = [\theta_1, \theta_2, \cdots, \theta_d]^\top$ are parameters that need to be learned, $\Bar{s}_k$ are centers of each RBFs kernel and $\sigma$ their bandwidth. In this experiment we set $\Sigma = \text{diag} (0.5, 0.5)$, $\sigma =0.5$, $d=441$  and $\Bar{s}_k = (x_k, y_k),  k=1, 2, \cdots, 441$ where $\Bar{s}_k$ forms a 2D uniform lattice with separation $0.5$.

The goal of the agent is to reach a goal position $s_{goal}=(8.5, 1.5)$ within the time horizon $T=20$, while avoiding the obstacles. To incentivize the agent to reach the goal, the reward function is formulated as the negative squared distance between the agent's current position and the goal position, i.e., $r(s_t)= -\lVert s_t - s_{goal} \rVert^2$. Moreover, the safe set $\mathcal{S}_\text{safe}$ is represented by the whole map/state space excluding regions of 5 obstacles in order to avoid the obstacles.

\begin{figure*}
\centering
\subfigure[Time Average Return]{
\begin{minipage}[t]{0.25\linewidth}
\centering
\includegraphics[width=\textwidth]{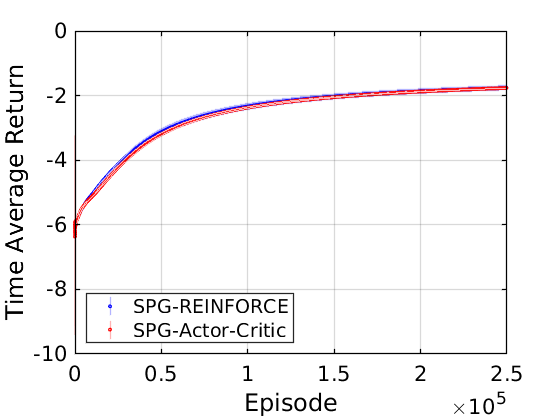}
%\caption{fig1}
\end{minipage}%
}%
\subfigure[Time Average Safety]{
\begin{minipage}[t]{0.25\linewidth}
\centering
\includegraphics[width=\textwidth]{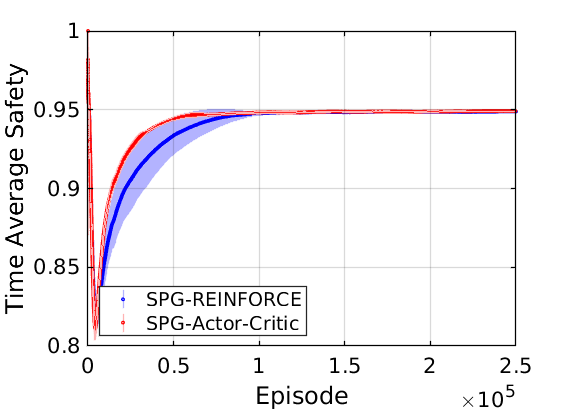}
%\caption{fig2}
\end{minipage}%
}%
\subfigure[Time Average Dual Variable (Lambda)]{
\begin{minipage}[t]{0.25\linewidth}
\centering
\includegraphics[width=\textwidth]{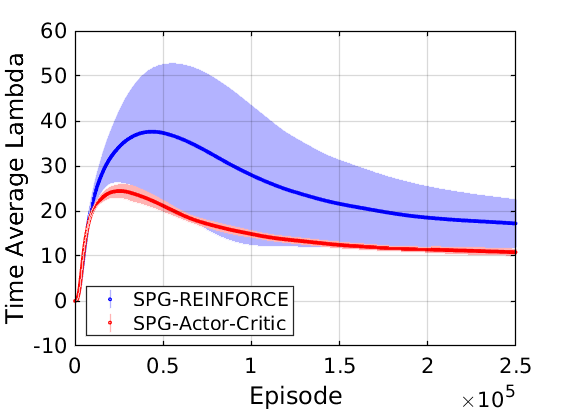}
%\caption{fig2}
\end{minipage}
}%
\subfigure[Visualization of $\hat{\mathbb{P}}$]{
\begin{minipage}[t]{0.25\linewidth}
\centering
\includegraphics[width=\textwidth]{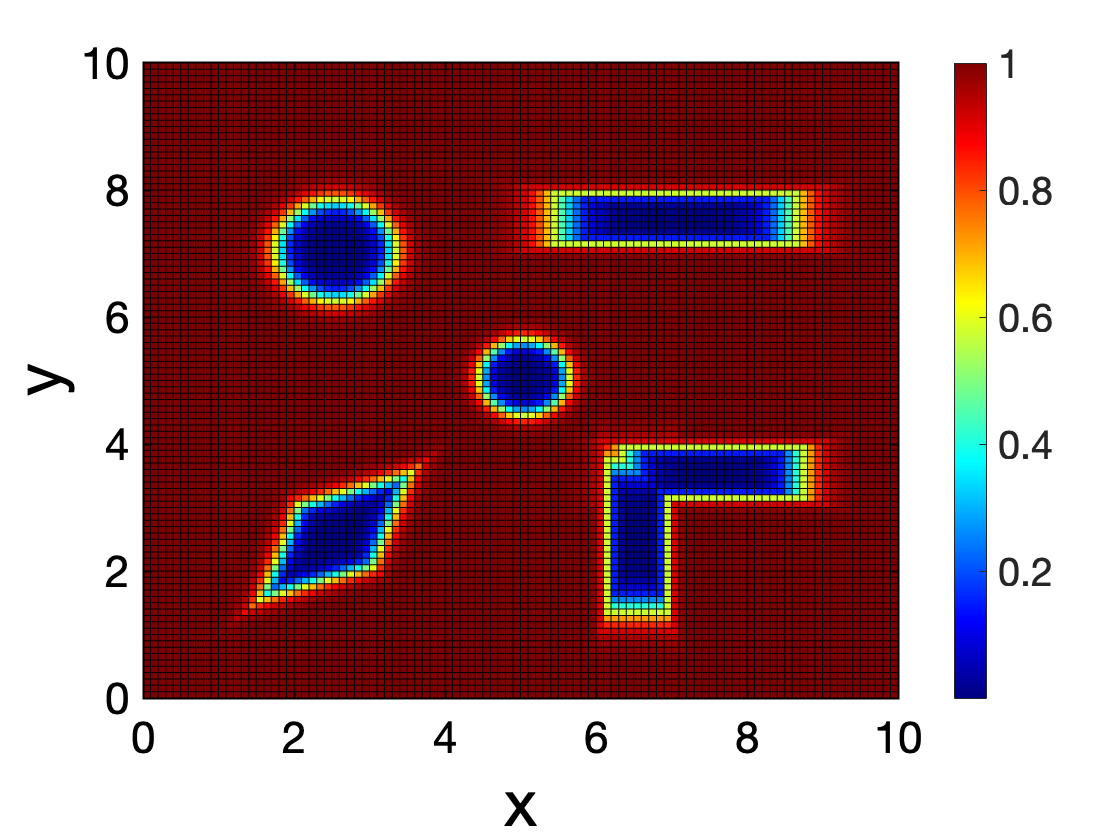}
%\caption{fig2}
\end{minipage}
}%
\centering
\caption{ Implementation of Algorithm~\ref{alg_pd} with the practical version in Step 4 for SPG-REINFORCE and SPG-Actor-Critic over 5 independent runs. Return (normalized by time horizon $T$), safety, and dual variable $\lambda$ are averaged over time. The hyper-parameters in both cases are set up as  $\eta_\theta = 0.02, \eta_\lambda = 0.002, 1-\delta=0.95, T=20$. Number of episodes is set to be 250,000. The solid line shows the mean and the shaded area depicts the standard deviation. }
\label{fig_primal_dual}
\end{figure*}

Having established the general RL settings as above, we are in the stage of validating Algorithm~\ref{alg_pd} with the two proposed SPGs. We implement Algorithm~\ref{alg_pd} employing both SPG-REINFORCE and SPG-Actor-Critic with the practical version in Step 4. In both implementations, the parameter $\theta$ is initialized as a zero vector in the $d$-dimensional real space. The initial value of $\lambda$ is set to be 0. The primal and dual step-sizes are set up as $\eta_\theta =0.02$ and $\eta_\lambda =0.002$. The desired level of safety, expressed as $1-\delta$, is set to be 0.95. The total number of iterations for the algorithm is established as 250,000.

To implement the stochastic SPG-Actor-Critic described in
\eqref{eqn_corro_ac_estimator}, estimating $\mathbb{E} [G_{t+1} |S_t, A_t]$ is the sole necessity due to the fact that $G_t^c$ and $\nabla_\theta \log \pi_\theta(A_t | S_t)$ can be computed easily from the historical data. By virtue of \eqref{def_G_cumulative_product}, we can rewrite $\mathbb{E} [G_{t+1} |S_t, A_t]$ as
\begin{align}
\mathbb{E} [G_{t+1} |S_t, A_t] = \mathbb{E} \left[ \prod_{u=t+1}^T\mathbbm{1}\left(S_u\in\mathcal{S}_{\text{safe}}\right) |S_t, A_t \right].
\end{align}
By the definition of indicator function, the previous equation is equivalent to
\begin{align}
\mathbb{E} [G_{t+1} |S_t, A_t] = \mathbb{P} \left(\bigcap\limits_{u=t+1}^{T} \{ S_u \in \mathcal{S}_\text{safe}\} |S_t, A_t  \right).
\end{align}
Therefore, $\mathbb{E} [G_{t+1} |S_t, A_t]$ is essentially the probability of agent being safe from $t+1$ to the end of the episode. Denote by $\hat{\mathbb{P}}$ the estimate of this probability. We next construct such an estimate as follows
\begin{align}\label{eqn_estimate_Gt1}
    \hat{\mathbb{P}} = \text{Sigmoid} \left(H_1 \left(\min_{i=1,\ldots,5} d_{i} -H_2 \right) \right),
\end{align}
where Sigmoid function ensures the result of \eqref{eqn_estimate_Gt1} to be between 0 and 1. $d_{i} (i=1, \ldots, 5)$ denotes the distance between the agent and the boundary of $i$-th obstacle. $H_1$ and $H_2$ represent two parameters that need to be learned within the activation function, which regulate the output of \eqref{eqn_estimate_Gt1} based on the distance between the agent and the obstacle. The underlying objective is to determine optimal values for $H_1$ and $H_2$ that lead to the estimate $\hat{\mathbb{P}}$ approaching 1 when the agent is located far away from any obstacles. However, as the agent moves closer to an obstacle, the estimate should gradually diminish. To learn $H_1$ and $H_2$, we minimize the loss function as follows with respect to $H_1$ and $H_2$, respectively
\begin{align}
    \text{loss} = \sum_{t=0}^{T} \left( \mathbb{E} [G_{t+1} |S_t, A_t] \!-\! \prod_{u=t+1}^T\mathbbm{1}\left(S_u\in\mathcal{S}_{\text{safe}}\right) \right)^2.
\end{align}

We now reach a point where we can evaluate the performance of Algorithm~\ref{alg_pd} with SPG-REINFORCE and SPG-Actor-Critic. Figure~\ref{fig_primal_dual} depicts the training curves in this setting across five independent runs. Figure~\ref{fig_primal_dual} (a), (b), and (c) illustrate the evolution of the time average return, safety (constraint), and dual variable $\lambda$, respectively. Depicted in red represents Algorithm~\ref{alg_pd} with SPG-Actor-Critic and depicted in blue employs SPG-REINFORCE. In all subfigures, the solid line shows the mean and the shaded area depicts the standard deviation. Both cases observe the convergence of the time average return with a comparable convergence rate and small standard deviation, settling at a value of approximately -2. This observation can be rationalized by the fact that the gradient of the reward component in both cases follows the identical expression as described in \eqref{eqn_classical_PG_estimator}.

The time average safety steadily converges towards the desired safety level, specifically $1-\delta=0.95$. Additionally, the time average dual variables $\lambda$ undergo a discernible convergent process as well, reaching a stable value of approximately 20 for SPG-REINFORCE and 10 for SPG-Actor-Critic. Notice that in Figure~\ref{fig_primal_dual} (b) and (c), SPG-Actor-Critic showcases a faster convergence rate and lower variance than SPG-REINFORCE, which substantiates our theoretical analysis in Theorem~\ref{theorem_lower_variance}. Figure~\ref{fig_primal_dual} (d) depicts the visualization of $\hat{\mathbb{P}}$ with the learned $H_1$ and $H_2$. As anticipated, the probability tends to converge to 1 when the agent is situated at a considerable distance from all obstacles. Conversely, the probability gradually diminishes and approaches 0 as the agent moves closer to any of the obstacles.

In order to showcase the effectiveness of Algorithm~\ref{alg_pd} in identifying optimal trajectories, we visualize the learned policy in Figure~\ref{configuration} (use SPG-REINFORCE as an example). Several sample trajectories have been displayed starting from (1, 1), (1, 9), (2, 5), and (8, 9) while adhering to the probabilistic safety constraints. This numerical example demonstrates the successful application of our proposed Safe Primal-Dual algorithm in a continuous navigation problem.

\begin{figure*}
\centering 
\subfigure[parameterized by \eqref{RBF}, $T=20$, 5 obstacles]
{
\begin{minipage}{8cm}
\centering    
\includegraphics[scale=0.4]{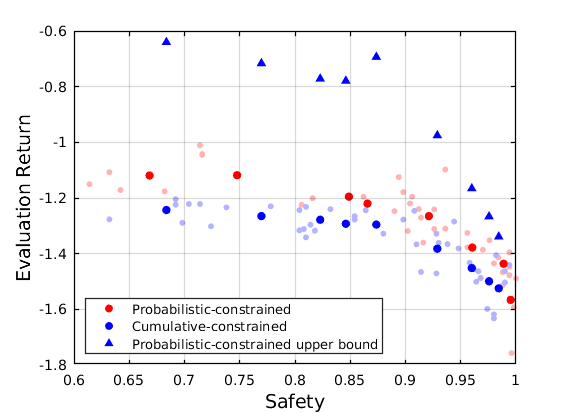}  
\end{minipage}
}
\subfigure[parameterized by \eqref{eqn_new_parameterization2}, $T=100$, 1 obstacle]
{
	\begin{minipage}{8cm}
	\centering 
	\includegraphics[scale=0.4]{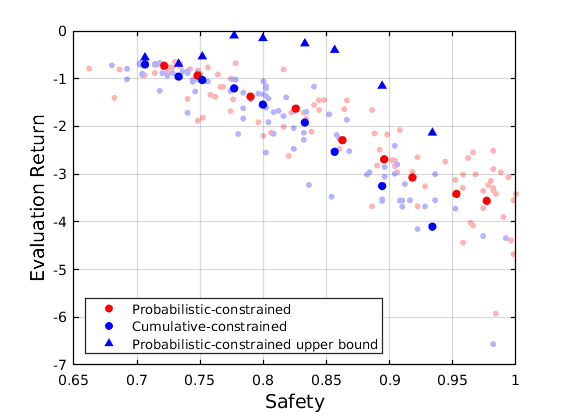}  
	\end{minipage}
}
\caption{
In both (a) and (b), depicted in blue dots is the safety-return for the cumulative-constrained formulation for different values ($\mu\in[40, 250]$ for (a) and $\mu\in[50, 3000]$ for (b)) and depicted in blue triangles is the corresponding upper bound given by Theorem~\ref{theorem_P_star_Ptilder_star}. Depicted in red dots is the safety-return for the probabilistic-constrained formulation for different values ($\lambda\in[10, 50]$ for (a) and $\lambda\in[50, 1500]$ for (b)). Both (a) and (b) observe that the probabilistic-constrained cases trace a better trade-off of evaluation return and safety, and are upper bounded by the blue triangles. Each large dot represents the mean under a fixed $\lambda$ or $\mu$. Each evaluation return (small dot) is averaged over 500 independent evaluations and the return is normalized by time horizon $T$.}
\label{fig_safety_reward_trade}
\end{figure*}
%
%
%

% \begin{figure*}
% \centering 
% \subfigure[]
% {
% \begin{minipage}{3cm}
% \centering    
% \includegraphics[scale=0.3]{figures/1return.png}  
% \end{minipage}
% }
% \subfigure[]
% {
% 	\begin{minipage}{3cm}
% 	\centering 
% 	\includegraphics[scale=0.3]{figures/2constraint.png}  
% 	\end{minipage}
% }
% \subfigure[]
% {
% 	\begin{minipage}{3cm}
% 	\centering 
% 	\includegraphics[scale=0.25]{figures/2constraint.png}  
% 	\end{minipage}
% }
% \subfigure[]
% {
% \begin{minipage}{3cm}
% \centering    
% \includegraphics[scale=0.1]{figures/4plotQ_3D.png}  
% \end{minipage}
% }
% \caption{Learning of the navigation problem employing SPG-REINFORCE and SPG-Actor-Critic with $\lambda=50$ over 5 runs. The solid line shows the mean and the shaded area depicts the standard deviation. (a). Evolution of the time averaged return with respect to episode.  (b). Evolution of the time averaged constraint with respect to episode. (c). Visualization of the estimation of $Q$ in SPG-Actor-Critic. SPG-Actor-Critic showcases a smaller variance than SPG-REINFORCE in terms of safety in (b).} 
% \label{fig_reinforce_ac}
% \end{figure*}

%
%
%
%
%
%
%
%
%
%
%
%
%
As discussed in Section~\ref{sec_problem_formulation}, $\mu$ and $\lambda$ trade-off performance in task completion (optimality) and safety. We run the regularization method under a fixed penalty with different weights. 
% For different values of $\lambda$ and $\mu$ different optimal step-sizes and total number of iterations can be used which results in faster or slower convergence of the algorithms. The worst case scenario requires $\eta_\theta = 0.004$ and needs 130,000 episodes. 
The safety in this example is the fraction of safe episodes (with all their states in the safe set $\mathcal{S_{\text{safe}}}$) that the trained policy yields over 500 independent episodes. The initial state is uniformly drawn from the safe set for each episode. Similarly, we evaluate the return by averaging the cumulative reward across the evaluation episodes. The scatter plot in Figure~\ref{fig_safety_reward_trade}(a) depicts the results of the safety and evaluation return under different values of $\mu \in [40, 250]$ and $\lambda \in [10, 50]$. Depicted in blue dots is the safety-return for the cumulative-constrained formulation (i.e., problem \eqref{eqn_problem2_mirror}), where unconstrained RL algorithms can directly apply through constructing a reward shaping $r_{\lambda}(S_t, A_t) = r(S_t, A_t) + \lambda /(T+1) \mathbbm{1} \left( S_t \in \mathcal{S}_\text{safe} \right)$. Depicted in blue triangles is the corresponding upper bound given by Theorem~\ref{theorem_P_star_Ptilder_star}, and 
depicted in red dots is the safety-return for the probabilistic-constrained formulation (i.e., problem \eqref{eqn_problem1}). Each large dot represents the mean under a fixed $\lambda$ or $\mu$. Observe that in both formulations, the higher the desired safety, the lower the average return is. Despite this common trend, for the same level of safety, the probabilistic-constrained problem yields an overall higher average return. This result is in accordance with the theoretical bounds provided in Theorem~\ref{theorem_P_star_Ptilder_star}.

Despite yielding a larger return, the probabilistic formulation still has upper bounds depicted as blue triangles in Figure~\ref{fig_safety_reward_trade}(a). These bounds are computed from $\hat{P}^\star (\text{blue dot}) + \hat{\lambda}^\star \delta T / (T+1)$. This result confirms the claim made in Theorem~\ref{theorem_P_star_Ptilder_star} as well. However, this bound could be conservative, as depicted in Figure~\ref{fig_safety_reward_trade}(a). On the other hand, the bound $\hat{\lambda}^\star \delta T / (T+1)$ demonstrates a tendency to contract as the level of safety increases. This intriguing observation can be attributed to the utilization of our well-designed parametrization (RBFs). As the level of safety gradually approaches 1 (i.e., $\delta \to 0$), it becomes evident that $\hat{\lambda}^\star$ does not need to be excessively large, indicating that the problem can be tackled relatively effortlessly. Consequently, the bound $\hat{\lambda}^\star \delta T / (T+1)$ exhibits a decreasing trend.

Next, we are intrigued by the behavior of the bound $\hat{\lambda}^\star \delta T / (T+1)$ in a ``challenging'' problem, or alternatively, when the policy's parametrization is not as optimal as RBFs. In Figure~\ref{fig_safety_reward_trade}(b), we consider the same navigation task except for setting $T=100$ and placing only one circular obstacle centered at (5, 5) with a radius $r=2$. In addition, we instead parameterize $\phi_\theta(s)$ as
\begin{align}\label{eqn_new_parameterization2}
    \phi_\theta(s) = \alpha(s) \theta_1^\top M(s) + (1-\alpha(s)) \theta_2^\top M(s),
\end{align}
where $\theta_1, \theta_2 \in \Real^{4 \times 2}$ denote the parameter matrices that need to be learned, and $M_s \in \Real^{4 \times 1}$ and $\alpha_s \in [0, 1]$ relates to the state $s=(x, y)$ as follows
\begin{align}\label{eqn_new_parameterization}
    &M(s) = [x-x_g, \, y-y_g, \, x-x_o, \, y-y_o]^\top,  \\
    &\alpha(s) = \text{Sigmoid} \left(H_1' \left(\sqrt{ (x-x_o)^2 + (y-y_o)^2}/r -H_2' \right) \right), \nonumber
\end{align}
where $(x_g, y_g), (x_o, y_o)$ represent the coordinates of the goal position and the obstacle, respectively. The hyper-parameters $H_1'$ and $H_2'$ serve a similar purpose to $H_1$ and $H_2$ in \eqref{eqn_estimate_Gt1}.
%
% In the experiment of Figure~\ref{fig_safety_reward_trade}(b), $H_1' = 50, H_2'=1.05$. 
By manipulating $\alpha(s)$, we ensure that it approaches 1 when the agent is far from the obstacle, and approaches 0 as the agent gets closer to the obstacle. Consequently, as $\alpha(s)$ approaches 1, the significance of $\theta_1$ becomes predominant in \eqref{eqn_new_parameterization2} which prioritizes maximizing rewards due to the agent being located far away from the obstacle. On the other hand, when $\alpha(s)$ approaches 0, $\theta_2$ takes precedence in \eqref{eqn_new_parameterization2}, prompting the agent to concentrate on avoiding the obstacle.

Figure~\ref{fig_safety_reward_trade}(b) depicts a similar observation as shown in Figure~\ref{fig_safety_reward_trade}(a). This empirical evidence further strengthens the validation of the theoretical bounds established in Theorem~\ref{theorem_P_star_Ptilder_star}. However, in contrast to Figure~\ref{fig_safety_reward_trade}(a), the upper bound $\hat{\lambda}^\star \delta T / (T+1)$ in Figure~\ref{fig_safety_reward_trade}(b) does not exhibit a noticeable trend of contraction as the safety level increases. The reason for this disparity lies in the suboptimal nature of the parametrization utilized in \eqref{eqn_new_parameterization2} compared to the more effective employment of RBFs, which consequently renders the navigation problem more challenging to solve. In comparison to RBFs, the novel parametrization in \eqref{eqn_new_parameterization2}  necessitates a higher value of $\hat{\lambda}^\star$ for a given safety level.
To conclude, we have verified the effectiveness of Algorithm~\ref{alg_pd} using a navigation task cluttered with obstacles, as demonstrated in Figure~\ref{configuration}. Moreover, the theoretical claims in Theorem~\ref{theorem_P_star_Ptilder_star} and Theorem~\ref{theorem_lower_variance} have been substantiated by the experiments implemented in Figure~\ref{fig_safety_reward_trade} and Figure~\ref{fig_primal_dual}, respectively.

%%%%%%%%%%%%%%%%%%%%%%%%
%%%%%%%%%%%%%%%%%%%%%%%%
%%%%%%%%%%%%%%%%%%%%%%%%
%%%%%%%%%%%%%%%%%%%%%%%%
%%%%%%%%%%%%%%%%%%%%%%%%
%%%%%%%%%%%%%%%%%%%%%%%%
%%%%%%%%%%%%%%%%%%%%%%%%
%%%%%%%%%%%%%%%%%%%%%%%%
%%%%%%%%%%%%%%%%%%%%%%%%
%%%%%%%%%%%%%%%%%%%%%%%%
%%%%%%%%%%%%%%%%%%%%%%%%
%%%%%%%%%%%%%%%%%%%%%%%%
%%%%%%%%%%%%%%%%%%%%%%%%
%%%%%%%%%%%%%%%%%%%%%%%%
%%%%%%%%%%%%%%%%%%%%%%%%
%%%%%%%%%%%%%%%%%%%%%%%%
%%%%%%%%%%%%%%%%%%%%%%%%
%%%%%%%%%%%%%%%%%%%%%%%%
%%%%%%%%%%%%%%%%%%%%%%%%
%%%%%%%%%%%%%%%%%%%%%%%%
%%%%%%%%%%%%%%%%%%%%%%%%
%%%%%%%%%%%%%%%%%%%%%%%%
%%%%%%%%%%%%%%%%%%%%%%%%
%%%%%%%%%%%%%%%%%%%%%%%%
%%%%%%%%%%%%%%%%%%%%%%%%
%%%%%%%%%%%%%%%%%%%%%%%%
%%%%%%%%%%%%%%%%%%%%%%%%
%%%%%%%%%%%%%%%%%%%%%%%%
%%%%%%%%%%%%%%%%%%%%%%%%
%%%%%%%%%%%%%%%%%%%%%%%%
%%%%%%%%%%%%%%%%%%%%%%%%
%%%%%%%%%%%%%%%%%%%%%%%%
%%%%%%%%%%%%%%%%%%%%%%%%
%%%%%%%%%%%%%%%%%%%%%%%%
%%%%%%%%%%%%%%%%%%%%%%%%
%%%%%%%%%%%%%%%%%%%%%%%%
%%%%%%%%%%%%%%%%%%%%%%%%
%%%%%%%%%%%%%%%%%%%%%%%%
%%%%%%%%%%%%%%%%%%%%%%%%
%%%%%%%%%%%%%%%%%%%%%%%%
%%%%%%%%%%%%%%%%%%%%%%%%
%%%%%%%%%%%%%%%%%%%%%%%%
%%%%%%%%%%%%%%%%%%%%%%%%
%%%%%%%%%%%%%%%%%%%%%%%%
%%%%%%%%%%%%%%%%%%%%%%%%
%%%%%%%%%%%%%%%%%%%%%%%%
%%%%%%%%%%%%%%%%%%%%%%%%
%%%%%%%%%%%%%%%%%%%%%%%%
%%%%%%%%%%%%%%%%%%%%%%%%
%%%%%%%%%%%%%%%%%%%%%%%%
%%%%%%%%%%%%%%%%%%%%%%%%
%%%%%%%%%%%%%%%%%%%%%%%%
%%%%%%%%%%%%%%%%%%%%%%%%
%%%%%%%%%%%%%%%%%%%%%%%%
%%%%%%%%%%%%%%%%%%%%%%%%
%%%%%%%%%%%%%%%%%%%%%%%%
%%%%%%%%%%%%%%%%%%%%%%%%
%%%%%%%%%%%%%%%%%%%%%%%%
%%%%%%%%%%%%%%%%%%%%%%%%
%%%%%%%%%%%%%%%%%%%%%%%%
%%%%%%%%%%%%%%%%%%%%%%%%
%%%%%%%%%%%%%%%%%%%%%%%%
%%%%%%%%%%%%%%%%%%%%%%%%
%%%%%%%%%%%%%%%%%%%%%%%%
\subsection{Lunar Lander}
Having numerically validated the efficacy of our previously proposed SPGs, Algorithm~\ref{alg_pd}, Theorem~\ref{theorem_P_star_Ptilder_star} and Theorem~\ref{theorem_lower_variance}, we further explore the performance of Algorithm~\ref{alg_pd} with SPG-REINFORCE and SPG-Actor-Critic in complex-dynamics scenarios. Here we consider a $\emph{LunarLander-v2}$ problem in OpenAI Gym toolkit~\cite{brockman2016openai}, which is visualized in Figure~\ref{lunar_lander}. The lander (spacecraft) subject to gravity, uses its thrusters to land on a landing pad (the region between two yellow flags) softly and fuel-efficiently. Surfaces on the “moon” are generated randomly in each episode, while the goal position is fixed at (0,0). The observation space encompasses a combination of six continuous state variables and two Booleans, specifically $(x, y, \dot{x}, \dot{y}, \theta, \dot{\theta}, \text{left\_leg}, \text{right\_leg})$ where $x, y$ denote the coordinates of the lander, $\dot{x}, \dot{y}$ denote the horizontal and vertical velocity, $\theta$ denotes the attitude, $ \dot{\theta}$ denotes angular velocity and the two Booleans indicate whether left/right leg is in contact with the ground. Moreover, there are four actions available in the action space: do nothing, fire main engine, fire left engine, fire right engine. Note that the lander can only select one of four actions at any given time. We refer to Figure~\ref{lunar_lander} for graphical representations.

% The dynamics of the lunar lander under the defined gravity $g$ and with reference to a Cartesian coordinate system with origin in the target point can be written as follows
% \begin{align}
%     &\ddot{x} = \frac{1}{m} ( \text{F}_\text{main} \sin(\theta)  - \text{F}_\text{side} \cos(\theta) ),  \nonumber \\
%     &\ddot{y} =  \frac{1}{m} ( \text{F}_\text{main} \cos(\theta) +  \text{F}_\text{side} \sin(\theta)  ) - g,  \nonumber \\
%     &\ddot{\theta} = \frac{ \text{F}_\text{side}  a}{I},
% \end{align}
% where $\text{F}_\text{main}$ is the main thrust, $\text{F}_\text{side}$ is net the side thrust, $m$ is the mass of the lander, $I$ is the lander's moment of inertia, and $a$ is the offset between side thrust axes and lander's center. 

In this example, we consider a Softmax policy
\begin{align}
\label{soft_max_policy}
\pi_\theta (a | s) =  \frac{ e^{\phi(s,a | \theta)}}{ \sum_{k=1}^4  e^{\phi(s,a_k | \theta)}},
\end{align}
where $\phi(s,a | \theta)$ is represented by the output of a fully-connected neural network (NN) and $\theta$ denotes the weights of NN. More specifically, NN that is used for solving $\emph{LunarLander-v2}$ consists of two hidden layers with 400 neurons in the first layer and 300 neurons in the second layer.

In each step of  $\emph{LunarLander-v2}$, as the lander approaches the landing pad, the reward increases, and as it moves away, the reward decreases. Similarly, the reward is increased/decreased the slower/faster the lander is moving, and the reward decreases as the lander tilts more.
In addition, each leg-ground contact earns a point of +10. Firing the main engine and side engine obtain -0.3 points and -0.03 points respectively for each frame. The episode finishes if the agent crashes or comes to rest, receiving additional -100/+100 points. The problem is deemed solved when achieving an averaged score over 200 points across 100 consecutive landing attempts. With this reward design, the agent is encouraged to reduce its distance to the landing pad, to decrease its speed smoothly, and to keep the angular speed at minimum to prevent rotation.
\begin{figure}
\centering
\includegraphics[width=0.7\columnwidth]{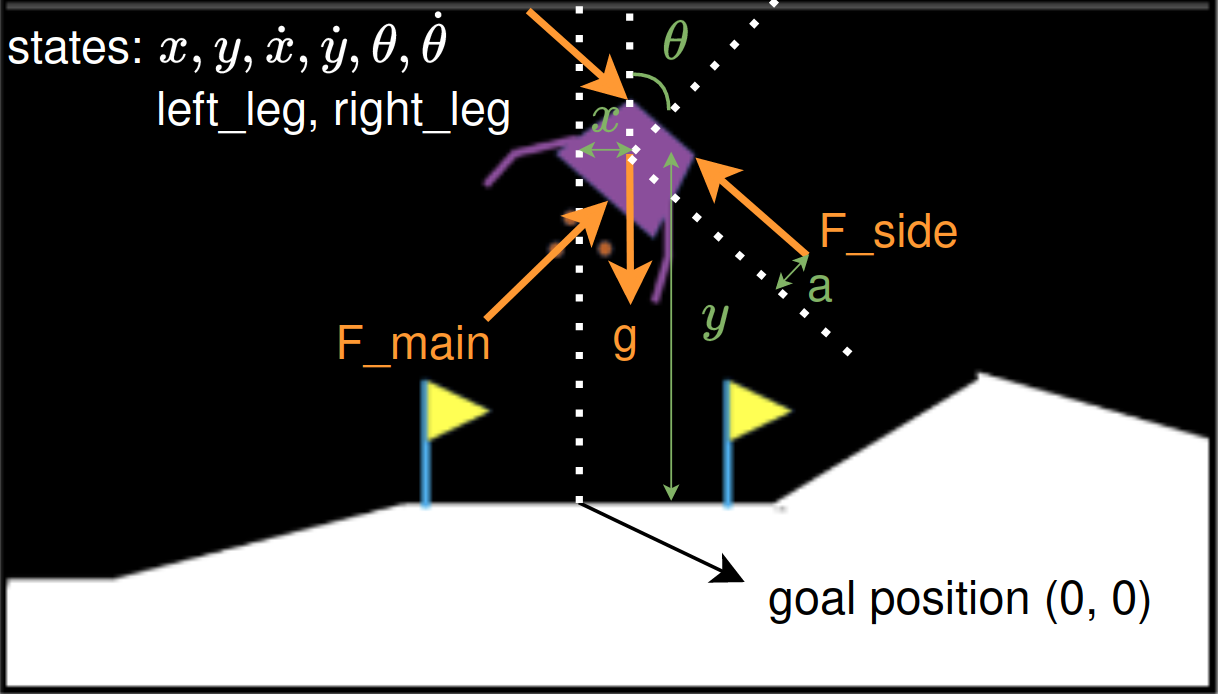}
\caption{Visualization of $\emph{LunarLander-v2}$. The state space consists of 6 continuous variables: horizontal coordinate $x$, vertical coordinate $y$, horizontal velocity $\dot{x}$, vertical velocity $\dot{y}$, angle $\theta$, angular velocity $\dot{\theta}$ and two Booleans for left and right legs indicating whether in contact with the ground. The goal position is fixed at (0, 0). The orange illustrates the main and side thrusts and gravity force configuration, in which $a$ is the offset between side thrust axes and lander's center.}
\label{lunar_lander}
\end{figure}

\begin{figure*}
\centering 
\subfigure[]
{
\begin{minipage}{8.5cm}
\centering    
\includegraphics[scale=0.4]{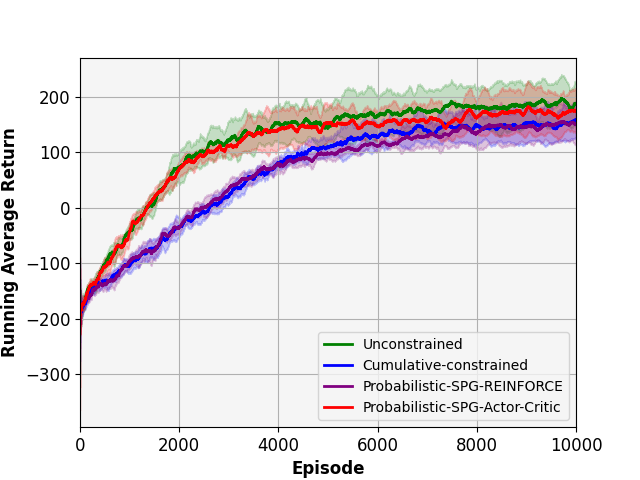}  
\end{minipage}
}
\subfigure[]
{
	\begin{minipage}{8.5cm}
	\centering 
	\includegraphics[scale=0.4]{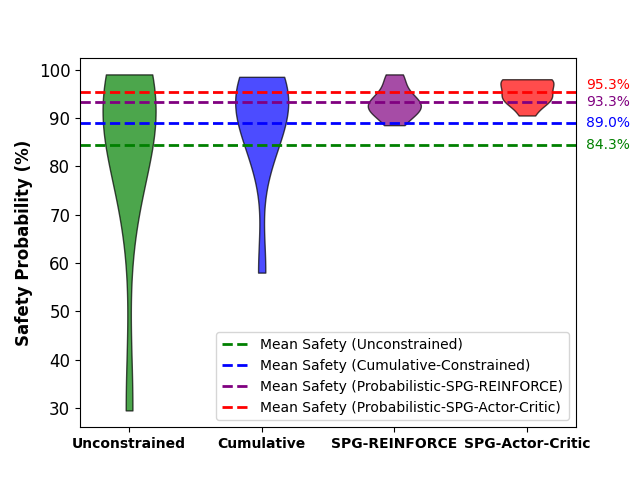}  
	\end{minipage}
}
\caption{Learning of $\emph{LunarLander-v2}$, averaged over 10 runs for each method. (a). Evolution of running average return over 100 consecutive landing attempts. The solid lines show the mean and the shaded areas depict the standard deviation. (b). Comparison  of safety probability. The safety probability of each run, denoted by the probability of $\{v < 0.9\}$ (desired velocity threshold), is computed by the number of safe evaluations divided by 200 independent evaluations, and the dash lines represent the mean of safety probability over 10 runs. Algorithm~\ref{alg_pd} and classical primal-dual method with $1-\delta=95\%$ are implemented in all experiments for the probabilistic-constrained and cumulative-constrained formulations, respectively.}
\label{fig_lunar_lander_curves}
\end{figure*}

% \begin{figure*}
% \centering 
% \subfigure[Evolution of $X$]
% {
% \begin{minipage}{0.33\linewidth}
% \centering    
% \includegraphics[scale=0.2]{figures/x_episode.png}  
% \end{minipage}
% }
% \subfigure[Evolution of $Y$]
% {
% 	\begin{minipage}{0.33\linewidth}
% 	\centering 
% 	\includegraphics[scale=0.2]{figures/y_episode.png}  
% 	\end{minipage}
% }
% \subfigure[Evolution of $V$]
% {
% 	\begin{minipage}{0.33\linewidth}
% 	\centering 
% 	\includegraphics[scale=0.2]{figures/v_episode.png}  
% 	\end{minipage}
% }
% \caption{Evaluation of $\emph{LunarLander-v2}$. Evolution of coordinates $x$ and $y$ of the lunar lander over time, averaged across 5 runs with each evaluating 100 episodes (500 evaluations in total). The solid line shows the mean and the shaded area depicts the standard deviation.}
% \label{fig_trajectory}
% \end{figure*}

\begin{figure*}[ht]
  \centering
  \subfigure[X-coordinate]{\includegraphics[width=0.32\linewidth]{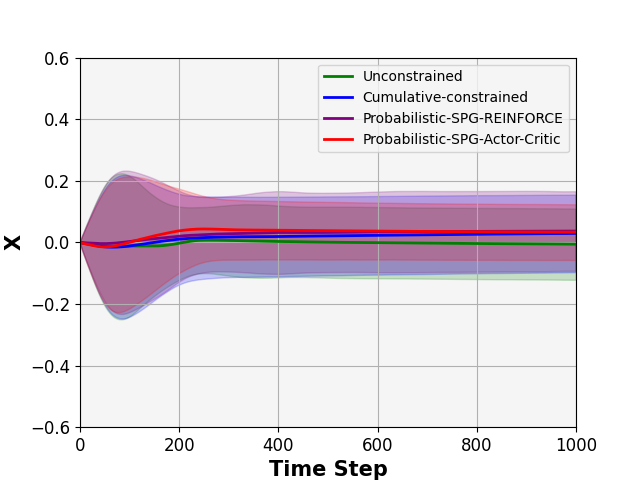}}
  \hfill
  \subfigure[Y-coordinate]{\includegraphics[width=0.32\linewidth]{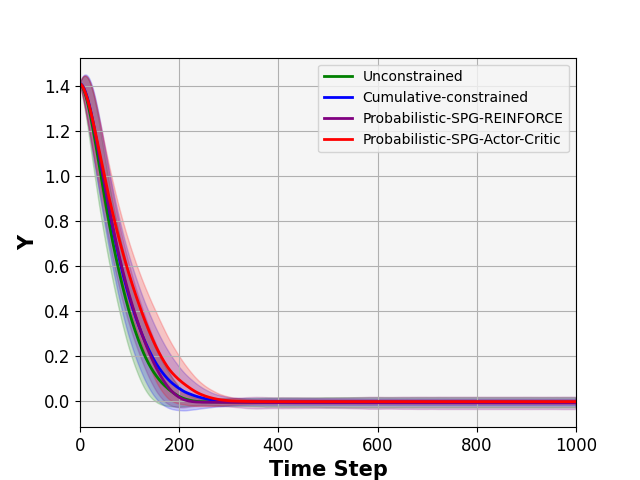}}
  \hfill
  \subfigure[Velocity]{\includegraphics[width=0.32\linewidth]{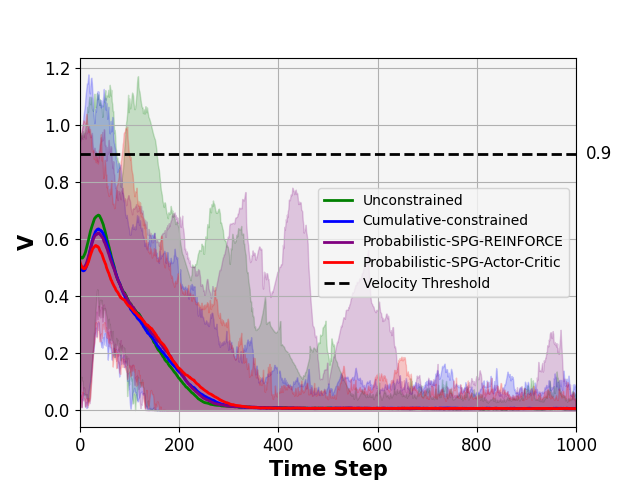}}
  \caption{Evaluation of \emph{LunarLander-v2}. Evolution of coordinates $X$, $Y$ and velocity $V$ of the lunar lander over time, averaged across 5 runs with each evaluating 100 episodes (500 evaluations in total). In (a) and (b), the solid lines show the mean and the shaded areas depict the standard deviation. In (c), the dash line represents the velocity threshold, and the solid lines represent the mean and the shaded areas demonstrate the minimum and maximum.}
  \label{fig_trajectory}
\end{figure*}

We depict in green in  Figure~\ref{fig_lunar_lander_curves} the training curve for a controller that focuses on maximizing the return without any safety consideration. While the problem can be considered ``solved'' it is worth pointing out that it often relies on large velocities. Accordingly, we constrain the velocity $V(t)=\sqrt{\dot{X}(t)^2+\dot{Y}(t)^2}$ to be less than 0.9 for all $t\in [1, 1000]$.

Subsequently, we test the effectiveness of Algorithm~\ref{alg_pd} with SPG-REINFORCE and SPG-Actor-Critic compared with the unconstrained case and the cumulative-constrained case. The learning curves are depicted in Figure~\ref{fig_lunar_lander_curves} (a), where the experiments are averaged over ten independent runs. We select the safety threshold of $1 - \delta=95\%$ and implement Algorithm~\ref{alg_pd} and classical primal-dual method in all experiments for the probabilistic-constrained and cumulative-constrained formulations, respectively. As demonstrated in Figure~\ref{fig_lunar_lander_curves}, the unconstrained case (green) successfully solves the task by achieving a running average return more than 200 points with the safety probability of $84.3\%$ on average. Safety probability is computed by the number of safe evaluations (the whole trajectory is safe) divided by 200 independent evaluations. Moreover, adding a probabilistic/cumulative constraint on the landing velocity (red, purple and blue) leads to higher safety probability but smaller average return. Despite that the cumulative-constrained case and the probabilistic-SPG-REINFORCE yield similar returns, the latter ($93.3\%$) maintains a higher safety probability on average than the former ($89.0\%$). It is worth highlighting that the probabilistic-SPG-Actor-Critic exhibits the highest value ($95.3\%$) and the lowest variance in terms of safety probability across ten independent runs, all the while securing the return that closely rivals the unconstrained scenario.

At the commencement of each episode, the lander's coordinates are randomly initialized around $X=0$ and $Y=1.4$, accompanied by non-zero linear and angular velocities. It starts with momentum and aims to navigate to reach the goal position at $(X, Y) = (0, 0)$. Figure~\ref{fig_trajectory} illustrates the success of the landing task. Figure~\ref{fig_trajectory} (a) and (b) provide an insight into the evolution of coordinates $X$ and $Y$ of the lander over time, and Figure~\ref{fig_trajectory} (c) depicts the evolution of the velocity $V$. They are averaged across 5 independent runs with each run consisting of 100 episodes. By evaluating a total of 500 episodes, we gain a comprehensive understanding of the lander's performance and its ability to consistently navigate towards the desired goal while constraining the velocity. Concretely, Figure~\ref{fig_trajectory} shows that SPG-Actor-Critic and SPG-REINFORCE can successfully accomplish the landing task while constraining the velocity well. In comparison, the two baselines (unconstrained and cumulative-constrained cases) lead to larger constraint violations. The lunar lander task showcases the successful application of Algorithm~\ref{alg_pd} with SPG-Actor-Critic and SPG-REINFORCE in a system characterized by complex dynamics.

%%%%%%%%%%%%%%%%%%%%%%%%
%%%%%%%%%%%%%%%%%%%%%%%%
%%%%%%%%%%%%%%%%%%%%%%%%
%%%%%%%%%%%%%%%%%%%%%%%%
%%%%%%%%%%%%%%%%%%%%%%%%
%%%%%%%%%%%%%%%%%%%%%%%%
%%%%%%%%%%%%%%%%%%%%%%%%
%%%%%%%%%%%%%%%%%%%%%%%%
%%%%%%%%%%%%%%%%%%%%%%%%
%%%%%%%%%%%%%%%%%%%%%%%%
%%%%%%%%%%%%%%%%%%%%%%%%
%%%%%%%%%%%%%%%%%%%%%%%%
%%%%%%%%%%%%%%%%%%%%%%%%
%%%%%%%%%%%%%%%%%%%%%%%%
%%%%%%%%%%%%%%%%%%%%%%%%
%%%%%%%%%%%%%%%%%%%%%%%%
%%%%%%%%%%%%%%%%%%%%%%%%
%%%%%%%%%%%%%%%%%%%%%%%%
%%%%%%%%%%%%%%%%%%%%%%%%
%%%%%%%%%%%%%%%%%%%%%%%%
%%%%%%%%%%%%%%%%%%%%%%%%
%%%%%%%%%%%%%%%%%%%%%%%%
%%%%%%%%%%%%%%%%%%%%%%%%
%%%%%%%%%%%%%%%%%%%%%%%%
%%%%%%%%%%%%%%%%%%%%%%%%
%%%%%%%%%%%%%%%%%%%%%%%%
%%%%%%%%%%%%%%%%%%%%%%%%
%%%%%%%%%%%%%%%%%%%%%%%%
%%%%%%%%%%%%%%%%%%%%%%%%
%%%%%%%%%%%%%%%%%%%%%%%%
%%%%%%%%%%%%%%%%%%%%%%%%
%%%%%%%%%%%%%%%%%%%%%%%%
%%%%%%%%%%%%%%%%%%%%%%%%
%%%%%%%%%%%%%%%%%%%%%%%%
%%%%%%%%%%%%%%%%%%%%%%%%
%%%%%%%%%%%%%%%%%%%%%%%%
%%%%%%%%%%%%%%%%%%%%%%%%
%%%%%%%%%%%%%%%%%%%%%%%%
%%%%%%%%%%%%%%%%%%%%%%%%
%%%%%%%%%%%%%%%%%%%%%%%%
%%%%%%%%%%%%%%%%%%%%%%%%
%%%%%%%%%%%%%%%%%%%%%%%%
%%%%%%%%%%%%%%%%%%%%%%%%
%%%%%%%%%%%%%%%%%%%%%%%%
%%%%%%%%%%%%%%%%%%%%%%%%
%%%%%%%%%%%%%%%%%%%%%%%%
%%%%%%%%%%%%%%%%%%%%%%%%
%%%%%%%%%%%%%%%%%%%%%%%%
%%%%%%%%%%%%%%%%%%%%%%%%
%%%%%%%%%%%%%%%%%%%%%%%%
%%%%%%%%%%%%%%%%%%%%%%%%
%%%%%%%%%%%%%%%%%%%%%%%%
%%%%%%%%%%%%%%%%%%%%%%%%
%%%%%%%%%%%%%%%%%%%%%%%%
%%%%%%%%%%%%%%%%%%%%%%%%
%%%%%%%%%%%%%%%%%%%%%%%%
%%%%%%%%%%%%%%%%%%%%%%%%
%%%%%%%%%%%%%%%%%%%%%%%%
%%%%%%%%%%%%%%%%%%%%%%%%
%%%%%%%%%%%%%%%%%%%%%%%%
%%%%%%%%%%%%%%%%%%%%%%%%
%%%%%%%%%%%%%%%%%%%%%%%%
%%%%%%%%%%%%%%%%%%%%%%%%
%%%%%%%%%%%%%%%%%%%%%%%%
\renewcommand{\arraystretch}{1.5}
\begin{table*}
	\centering
	\fontsize{7}{8}\selectfont
	\begin{threeparttable}
		\caption{
  Return and safety of three methods (unconstrained, cumulative-constrained, and probabilistic-constrained) in two Safety Gym environments: PointGoal1 and CarGoal1.}
		\label{table_safe_gym}
		\begin{tabular}{ccc|ccc}
        \hline
        \hline
			PointGoal1 & Return & Safety ($\%$) & CarGoal1 & Return & Safety ($\%$) \cr
        \hline
             Unconstrained (DDPG)  & ${\bf 14.79 \pm 5.15}$ & $80.59 \pm 14.00$ &
                    
			Unconstrained (DDPG) & ${\bf 17.65 \pm 2.93}$  & $79.88 \pm 39.94$ \cr
   
             Cumulative-constrained (DDPG-Lagrangian) & $5.63 \pm 5.90$ & $89.39 \pm 14.43$ &
            
            Cumulative-constrained (DDPG-Lagrangian) & $7.85 \pm 7.42$ & $89.38 \pm 29.81$ \cr
            
             Probabilistic-constrained (SPG-Actor-Critic) & $6.76 \pm 6.08$ & ${\bf 94.30 \pm 5.34}$ &
            
			Probabilistic-constrained (SPG-Actor-Critic) & ${14.62 \pm 2.32}$ & ${\bf 99.06 \pm 2.02}$ \cr
        \hline
        \hline
		\end{tabular}
	\end{threeparttable}
\end{table*}

\subsection{Safety Gym}
To further substantiate our proposed SPGs and Theorem~\ref{theorem_P_star_Ptilder_star} using safe RL benchmark simulators, we consider the navigation tasks PointGoal1 and CarGoal1 in the Safety Gym (see~\cite{ray2019benchmarking} for details). %Here we consider two environments: PointGoal1 and CarGoal1 (see \cite{ray2019benchmarking} for details). 
For the probabilistic-constrained case, we employ our proposed SPG-Actor-Critic, while the baselines include two state-of-the-art RL algorithms: DDPG~\cite{lillicrap2015continuous} (unconstrained case) and DDPG-Lagrangian~\cite{chow2018lyapunov} (cumulative-constrained case). For each algorithm, we train the policy across 10 independent runs and obtain 10 final policies. Then each policy goes through 5 independent evaluations. The numerical results are summarized and compared in Table~\ref{table_safe_gym}. The results validate that our proposed SPG-Actor-Critic finds a better return-safety trade-off than the two baseline state-of-the-art algorithms. More specifically, the unconstrained case (DDPG) can achieve the highest return in both environments with, however, low safety. Notice that imposing constraints can improve the safety. Indeed, the probabilistic-constrained case (our method  SPG-Actor-Critic) can achieve both higher return and higher safety than the cumulative-constrained case (DDPG-Lagrangian), which can be explained by our Theorem~\ref{theorem_P_star_Ptilder_star}: the solution of problem~\eqref{eqn_problem1} outperforms problem~\eqref{eqn_problem2_mirror}. In summary, the experiments in Safety Gym show that our proposed SPG-Actor-Critic outperforms cutting-edge cumulative-constrained algorithms in the standard safe RL benchmarks.

%% file: conclusions.tex
%!TEX root = root.tex
\section{Conclusions}\label{sec_conclusions}
In this work, we have studied the problem of learning safe policies under probabilistic constraints. Concretely, a safe policy is defined as one that guarantees, with high probability, that the agent remains in a desired safe set across the whole trajectory. We have established theoretical bounds on the improvement in the optimality-safety trade-off that the probabilistic-constrained formulation provides compared to the setting of cumulative constraints. We have also provided expressions (SPG-REINFORCE and SPG-Actor-Critic) for the gradient of the probabilistic constraint, along with the Safe Primal-Dual algorithm, for which we have established the guarantee of the convergence, near-optimality on average, and feasibility on average. The Safe Primal-Dual algorithm with SPGs as well as our theoretical findings in this work is substantiated in three numerical experiments: (i) a navigation problem in cluttered environments, (ii) a lunar lander problem constraining the maximum velocity, and (iii) standard safe RL benchmarks in Safety Gym. Future work includes characterizing the running time and sample complexity of the Safe Primal-Dual algorithm, and applying SPGs to other complicated dynamic systems.

%Regardless of the constraint considered, a natural way to solve RL problems is to find the corresponding gradients of the objective functions and constraints. Despite that classic policy gradient algorithms apply in cumulative-constrained setting directly, the gradients for solving probabilistic formulations are not available. In this work, we have provided the first explicit gradient expressions for the probabilistic constraint. Moreover, we introduce two methods that harness the potential of the proposed SPGs, along with analyses that covers convergence, optimality, and feasibility. In addition, We have showcased the applicability of utilizing these methods for solving continuous navigation problems within cluttered environments and tackling the lunar lander problem, characterized by intricate and complex dynamics. We have confirmed empirically our theoretical results regarding the safety-performance trade-off that both notions of safety induce, and the lower variance that SPG-Actor-Critic prompts.

%Importantly, the integration of SPG-REINFORCE and SPG-Actor-Critic with state-of-the-art algorithms, investigating their convergence rates, data-efficiency, and their applicability to solving other RL problems with even more intricate dynamics, lies beyond the scope of this work.

%% file: appendix.tex
\section{Appendix}\label{Appendix}
%%%%%%%%%%%%%%%%%%
%%%%%%%%%%%%%%%%%%
%%%%%%%%%%%%%%%%%%
%%%%%%%%%%%%%%%%%%
\subsection{Technical Lemmas for the Proof of Theorem~\ref{theorem_P_star_Ptilder_star}}
\label{appendix_proposition_figure}
\begin{lemma}
\label{lemma_bound_zero_duality_gap}
Let hypotheses of Theorem~\ref{theorem_P_star_Ptilder_star} hold. Consider the function $\tilde{P}^\star(\xi)$ defined in \eqref{eqn_problem2}. Let $\xi_0,\xi_1\in\mathbb{R}$. Let $\tilde{\lambda}^\star (\xi_0)$ be the dual optimal solution to \eqref{eqn_problem2} with $\xi=\xi_0$, defined tantamount to \eqref{eqn_dual_solution}.  It holds that
\begin{align}
    \tilde{P}^\star(\xi_1) \leq \tilde{P}^\star(\xi_0) + \tilde{\lambda}^\star (\xi_0) (\xi_0 - \xi_1).
\end{align}
\end{lemma}
\begin{proof}
Recall definition of the dual problem associated to~\eqref{eqn_problem2}
\begin{align}\label{eqn_min_max}
    \tilde{D}^\star(\xi)= \min\limits_{ \tilde{\lambda}\in\mathbb{R}_+} \, \max\limits_{\theta\in\mathbb{R}^d} \, V(\theta) +  \tilde{\lambda} (V_c(\theta) - \xi).
\end{align}
It follows from Theorem 3 in \cite{paternain2022safe} that zero duality gap holds for problem \eqref{eqn_problem2}
\begin{align}\label{zero_duality_gap_1}
    \tilde{P}^\star \! (\xi_1) \! = \! \tilde{D}^\star \! (\xi_1) \! = \! V\!(\theta^\star (\xi_1)) \! + \! \tilde{\lambda}^\star (\xi_1) (V_c(\theta^\star (\xi_1)) \!-\! \xi_1),
\end{align}
where $(\theta^\star (\xi_1), \tilde{\lambda}^\star (\xi_1))$ denote the primal-dual optimal solution of \eqref{eqn_problem2} with $\xi = \xi_1$. Likewise, we can also write 
\begin{align}\label{zero_duality_gap_5}
    \tilde{P}^\star (\xi_0) = V(\theta^\star (\xi_0)) +  \tilde{\lambda}^\star (\xi_0) (V_c(\theta^\star (\xi_0)) - \xi_0),
\end{align}
where the primal-dual solution with respect to $\xi_0$ is denoted by $(\theta^\star (\xi_0), \tilde{\lambda}^\star (\xi_0))$. 
By definition of $\tilde{\lambda}^\star (\xi_1)$, i.e., the minimizer of \eqref{eqn_min_max} with $\xi=\xi_1$, it follows that for any $\lambda>0$ we have 
\begin{align}
    \tilde{P}^\star(\xi_1)&=V(\theta^\star (\xi_1)) +   \tilde{\lambda}^\star (\xi_1) (V_c(\theta^\star (\xi_1)) - \xi_1) \nonumber \\ 
    &\leq  V(\theta^\star (\xi_1)) +   \lambda (V_c(\theta^\star (\xi_1)) - \xi_1).
\end{align}
In particular, this holds for $\lambda= \tilde{\lambda}^\star(\xi_0)$
\begin{align}\label{eqn_temp11}
    \tilde{P}^\star(\xi_1) \leq  V(\theta^\star (\xi_1)) +   \tilde{\lambda}^\star(\xi_0) (V_c(\theta^\star (\xi_1)) - \xi_1).
\end{align}

By adding and subtracting $\tilde{\lambda}^\star (\xi_0) \, \xi_0$ to \eqref{eqn_temp11} yields
\begin{align}\label{zero_duality_gap_2}
   \tilde{P}^\star (\xi_1) &\leq V(\theta^\star (\xi_1)) +  \tilde{\lambda}^\star (\xi_0) (V_c(\theta^\star (\xi_1)) - \xi_0) \nonumber \\
   &+  \tilde{\lambda}^\star (\xi_0) (\xi_0 - \xi_1).
\end{align}

Likewise, $\theta^\star (\xi_0)$ is the primal maximizer of the Lagrangian with $\xi = \xi_0$
\begin{equation}
\theta^\star(\xi_0) = \argmax_{\theta\in\Real^d} V(\theta) + \tilde{\lambda}^\star(\xi_0)\left(V_c(\theta)-\xi_0\right).
\end{equation}
Thus $\tilde{P}^\star(\xi_1)$ in \eqref{zero_duality_gap_2} is upper bounded as
\begin{align}\label{zero_duality_gap_3}
    \tilde{P}^\star(\xi_1) &\leq V(\theta^\star (\xi_0)) +   \tilde{\lambda}^\star (\xi_0) (V_c(\theta^\star (\xi_0)) - \xi_0) \nonumber\\
    &+ \tilde{\lambda}^\star (\xi_0) (\xi_0 - \xi_1).
\end{align}
Substituting \eqref{zero_duality_gap_5} into \eqref{zero_duality_gap_3} reduces to
\begin{align}
    \tilde{P}^\star(\xi_1) \leq \tilde{P}^\star(\xi_0) + \tilde{\lambda}^\star (\xi_0) (\xi_0 - \xi_1).
\end{align}
This completes the proof of Lemma~\ref{lemma_bound_zero_duality_gap}.
\end{proof}

\begin{lemma}\label{proposition_figure}
Let $\hat{\mathcal{F}}$, $\mathcal{F}$ and $\bar{\mathcal{F}}$ represent the feasible sets of problem~\eqref{eqn_problem2_mirror}, problem~\eqref{eqn_problem1} and problem~\eqref{eqn_problem_mirrormirror}, respectively. It holds that $\hat{\mathcal{F}} \subseteq \mathcal{F} \subseteq \bar{\mathcal{F}}$, as depicted in Figure~\ref{fig_three_constraints}.
\end{lemma}
\begin{figure}[htbp]
\centering
\includegraphics[width=3.5cm]{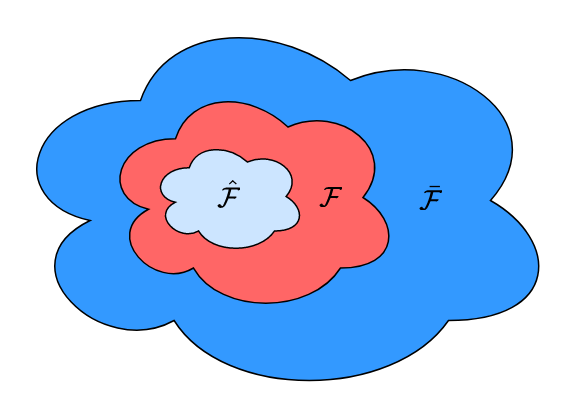}
\caption{\small{The illustration for feasible sets of problem~\eqref{eqn_problem2_mirror}--$\hat{\mathcal{F}}$ (light blue), problem~\eqref{eqn_problem1}--$\mathcal{F}$ (red), and problem~\eqref{eqn_problem_mirrormirror}--$\bar{\mathcal{F}}$ (navy blue).}}
\label{fig_three_constraints}
\end{figure}
\begin{proof}
We start by proving the leftmost inclusion. Denote by $\hat{\theta}$ the any feasible solution to problem~\eqref{eqn_problem2_mirror}, i.e., $\hat{\theta} \in \hat{\mathcal{F}}$. By virtue of Proposition~\ref{proposition_U_cP_1-delta} the policy $\pi_{\hat{\theta}}$ is $(1-\delta)$ safe in the sense of Definition~\ref{definition_safety}. In turn, this means that $\hat{\theta}$ is a feasible solution to problem~\eqref{eqn_problem1} as well. Then it holds that $\hat{\theta} \in \mathcal{F}$ for $\forall \hat{\theta} \in \hat{\mathcal{F}}$, thus $\hat{\mathcal{F}} \subseteq \mathcal{F}$.

We now focus on proving the second inclusion. Denote by $\bar{\theta}$ any point in $\mathcal{F}$. $\bar{\theta}$ is thus a feasible solution to problem~\eqref{eqn_problem1}
\begin{equation}
\mathbb{P} \left(\bigcap\limits_{t=0}^{T} \{ S_t \in \mathcal{S}_\text{safe}\} \mid \pi_{\bar{\theta}} \right) \geq 1-\delta.
\end{equation}
Observe that the previous inequality is equivalent to
\begin{equation}\label{eqn_lemma1_aux}
\mathbb{P} \left(\sum\limits_{t=0}^{T} \mathbbm{1} \left( S_t \in \mathcal{S}_\text{safe} \right) = T+1 \mid \pi_{\bar{\theta}} \right) \geq 1-\delta.
\end{equation}
Indeed, for the agent to belong to $\mathcal{S}_{\text{safe}}$ for all times, all the indicator functions in \eqref{eqn_lemma1_aux} need to take the value of 1. Since $\sum_{t=0}^T \mathbbm{1} \left( S_t \in \mathcal{S}_\text{safe} \right)$ is a non-negative random variable, it follows that 
\begin{align}\label{eqn_lemma1_aux2}
&\mathbb{E} \left[\sum\limits_{t=0}^{T} \mathbbm{1} \left( S_t \in \mathcal{S}_\text{safe} \right) |\pi_{\bar{\theta}} \right] \\  &\geq \mathbb{P} \left(\sum\limits_{t=0}^{T} \mathbbm{1} \left( S_t \in \mathcal{S}_\text{safe} \right) = T+1 |\pi_{\bar{\theta}} \right) (T+1). \nonumber
\end{align}
Combining \eqref{eqn_lemma1_aux} and \eqref{eqn_lemma1_aux2}, it follows that
\begin{align}
    \mathbb{E} \left[\frac{1}{T+1} \sum\limits_{t=0}^{T} \mathbbm{1} \left( S_t \in \mathcal{S}_\text{safe} \right) |\pi_{\bar{\theta}} \right] \geq 1-\delta.
\end{align}
Hence, $\bar{\theta}$ is a feasible point in $\bar{\mathcal{F}}$ for $\forall \bar{\theta} \in \mathcal{F}$, i.e., $\mathcal{F} \subseteq \bar{\mathcal{F}}$. This completes the proof of Lemma~\ref{proposition_figure}.
\end{proof}

%%%%%%%%%%%%%%%%%%
%%%%%%%%%%%%%%%%%%
%%%%%%%%%%%%%%%%%%
%%%%%%%%%%%%%%%%%%

%%%%%%%%%%%%%%%%%%
%%%%%%%%%%%%%%%%%%
%%%%%%%%%%%%%%%%%%
%%%%%%%%%%%%%%%%%%
\subsection{Proof of Theorem~\ref{theorem_safe_policy_gradient}}
\label{appendix_theorem_safe_policy_gradient}

We proceed by presenting two technical lemmas. 
%
% We proceed by presenting and proving the following two technical lemmas (Lemma~\ref{lemma_safe_policy_gradient_G1} and Lemma~\ref{lemma_nabla_E_G1_S0_GT_ST-1}).
%
%
\begin{lemma}[\cite{chen2023policy}, Lemma 1]
\label{lemma_safe_policy_gradient_G1}
Given $S_{t-1} \in \mathcal{S}_\text{safe}$ and $G_{t}, t=1,2,\cdots, T-1$ defined in \eqref{def_G_cumulative_product}, it holds that
\begin{align}\label{eqn_recursive_gradient}
\nabla_\theta\mathbb{E}\left[G_{t}\mid S_{t-1}\right] &= \mathbb{E}\left[\nabla_\theta\mathbb{E}\left[G_{t+1}\! \mid \! S_{t}\right]\mathbbm{1}\left(S_{t}\in\mathcal{S}_{\text{safe}}\right) \! \mid \! S_{t-1} \right] \nonumber \\
&+ \mathbb{E}\left[G_{t}\nabla_{\theta}\log\pi_\theta(A_{t-1} \!\!\mid\!\! S_{t-1}) \!\!\mid\!\! S_{t-1}\right].
\end{align}
\end{lemma}
\begin{lemma}[\cite{chen2023policy}, Lemma 2]
\label{lemma_nabla_E_G1_S0_GT_ST-1}
Given $S_{t-1} \in \mathcal{S}_\text{safe}$ and $G_{t}, t=1,2,\cdots, T-1$ defined in \eqref{def_G_cumulative_product}, it holds that
\begin{align}\label{eqn__nabla_E_G1_S0_GT_ST-1}
    \nabla_\theta\mathbb{E}\left[G_1\mid S_0\right] &=\sum\limits_{t=0}^{T-2}\mathbb{E}\left[G_1\nabla_{\theta}\log\pi_\theta(A_t\mid S_t)\mid S_0\right]  \\ &+\mathbb{E}\left[\nabla_\theta\mathbb{E}\left[G_T \! \mid \! S_{T-1}\right]\! \prod_{t=1}^{T-1} \! \mathbbm{1}\left(\! S_{t}\in\mathcal{S}_{\text{safe}} \!\right) \! \mid \! S_0\right].  \nonumber
\end{align}
\end{lemma}
To prove Theorem~\ref{theorem_safe_policy_gradient}, rewrite the probability of remaining safe in terms of $G_0$ defined in \eqref{def_G_cumulative_product}. By definition of probability we have
%
% \begin{align}
%     &\mathbb{P} \left(\bigcap\limits_{t=0}^{T} \{ S_t \in \mathcal{S}_\text{safe}\} |\pi_\theta, S_0 \right) \nonumber \\
%     &= \mathbb{E} \left[\mathbbm{1} \left(\bigcap\limits_{t=0}^{T} \{ S_t \in \mathcal{S}_\text{safe}  \}\right) |\pi_\theta, S_0 \right].
% \end{align}
%
\begin{align}
    \mathbb{P} \! \left( \bigcap\limits_{t=0}^{T} \{ S_t \in \mathcal{S}_\text{safe}\} |\pi_\theta, S_0 \! \right) \!\!=\! \mathbb{E} \! \left[\mathbbm{1} \! \left(\bigcap\limits_{t=0}^{T} \{ S_t \in \mathcal{S}_\text{safe} \! \}\right) \! |\pi_\theta, S_0 \right].
\end{align}
Note that the indicator function in the previous expression takes the value one, if and only if each $S_t\in\mathcal{S}_{\text{safe}}$. Hence, it is possible to rewrite the previous expression in terms of the product of indicator functions of states satisfying the safety condition at each time
\begin{align}\label{eqn_safe_policy_gradient_G0}
    \mathbb{P} \left(\bigcap\limits_{t=0}^{T} \{ S_t \in \mathcal{S}_\text{safe}\} |\pi_\theta, S_0 \right) &= \mathbb{E} \left[\! \prod\limits_{t=0}^{T} \! \mathbbm{1} (S_t \in \mathcal{S}_\text{safe}) |\pi_\theta, S_0 \right] \nonumber \\
    &= \mathbb{E}\left[G_0 | S_0\right],
\end{align}
where $\pi_\theta$ is omitted in the last equation for simplicity. By virtue of $S_0 \in \mathcal{S}_\text{safe}$, we obtain $ \mathbb{E}[G_0 | S_0] = \mathbb{E}[G_1 \cdot \mathbbm{1} (S_0 \in \mathcal{S}_\text{safe}) | S_0]=\mathbb{E}[G_1 | S_0].$ Then, using \eqref{eqn_safe_policy_gradient_G0}, the gradient of the probability of remaining safe reduces to
\begin{equation}\label{eqn_safe_policy_gradient_G1}
    \nabla_\theta \mathbb{P} \left(\bigcap\limits_{t=0}^{T} \{ S_t \in \mathcal{S}_\text{safe}\} |\pi_\theta, S_0 \right)=\nabla_\theta \mathbb{E}\left[G_1 | S_0\right].
\end{equation}
In Lemma~\ref{lemma_safe_policy_gradient_G1} we derive a recursive relationship for the gradient of $\mathbb{E}\left[G_t\mid S_{t-1}\right], t=1,2,\cdots, T-1 $. By virtue of Lemma~\ref{lemma_nabla_E_G1_S0_GT_ST-1}, to complete the proof of the result it suffices to establish that 
\begin{align}\label{eqn_thing_to_show}
    &\mathbb{E}\left[\nabla_\theta\mathbb{E}\left[G_T\mid S_{T-1}\right]\prod_{t=1}^{T-1}\mathbbm{1}\left(S_{t}\in\mathcal{S}_{\text{safe}}\right) \mid S_0\right] \nonumber \\
    &= \mathbb{E}\left[G_1\nabla_\theta \log \pi_{\theta}(A_{T-1} \mid S_{T-1})\mid S_0\right]. 
\end{align}
We establish this result next. Let us start by working with the gradient of the inner expectation on the left-hand side of \eqref{eqn_thing_to_show}. 

Using the fact that $G_T = \mathbbm{1}\left(S_T\in\mathcal{S}_{\text{safe}}\right)$ and the definition of expectation one can write $\nabla_\theta\mathbb{E}\left[G_T\mid S_{T-1}\right]$ in the left hand side of the previous expression as
\begin{equation}\label{eqn_nabla_G_T}
   \nabla_\theta\mathbb{E}\left[G_T \!\!\mid\!\! S_{T-1}\right] \!=\! \nabla_\theta \!\! \int_{\mathcal{S}} \! \mathbbm{1} (s_T \in \mathcal{S}_\text{safe}) p(s_T | S_{T-1}) \, ds_T, 
\end{equation}
where $p\left(s_T\mid S_{T-1}\right)$ denotes the conditional probability of $S_T$ given $S_{T-1}$. Marginalizing the probability distribution it follows that 
\begin{align}
    p(\!s_T | S_{T-1}\!) \! = \! \!\!\int_{\mathcal{A}} \!\!\!\! p(\! s_T | S_{T-1}, \! a_{T-1} \!)  \pi_{\theta}(\! a_{T-1} | S_{T-1} \!) da_{T-1}.
\end{align}
Consequently, \eqref{eqn_nabla_G_T} can be converted to
\begin{align}
   &\nabla_\theta\mathbb{E}\left[G_T\mid S_{T-1}\right] \nonumber \\
   &= \nabla_\theta\int_{\mathcal{S}\times \mathcal{A}} \mathbbm{1}\left(s_T \in \mathcal{S}_{\text{safe}}\right) p(s_T \mid S_{T-1}, \, a_{T-1}) \nonumber \\
   &\quad \quad \quad \quad \quad \quad \pi_{\theta}(a_{T-1}\mid S_{T-1}) \, ds_T da_{T-1}.
\end{align}
Note that in the previous expression, the only term dependent on $\theta$ is the policy, hence we have that
\begin{align}\label{eqn_temp22}
   \nabla_\theta\mathbb{E}\left[G_T\mid S_{T-1}\right]  =&\int_{\mathcal{S}\times \mathcal{A}}\! \mathbbm{1}\left(s_T \in \mathcal{S}_{\text{safe}}\right) p(s_T \! \mid \! S_{T-1}, \, a_{T-1}) \nonumber \\
   &\nabla_\theta\pi_{\theta}(a_{T-1}\mid S_{T-1}) \, ds_T da_{T-1}. 
\end{align}
Applying the ``log-trick'' to the right-hand side of \eqref{eqn_temp22} yields
\begin{align}
   &\nabla_\theta\mathbb{E}\left[G_T\mid S_{T-1}\right] \nonumber \\ &=\int_{\mathcal{S}\times \mathcal{A}}\mathbbm{1}\left(s_T \in \mathcal{S}_{\text{safe}}\right) p(s_T \mid S_{T-1}, \, a_{T-1})  \\
   &\quad~\pi_{\theta}(a_{T-1}\mid S_{T-1}) \nabla_\theta\log\pi_{\theta}(a_{T-1}\mid S_{T-1}) \, ds_T da_{T-1}.  \nonumber
   \end{align}
Since $p(s_T| S_{T-1},a_{T-1})\pi_\theta(a_{T-1}| S_{T-1}) \!= \! p(s_T,a_{T-1}| S_{T-1})$ is the joint probability distribution of $S_{T}$ and $A_{T-1}$ given $S_{T-1}$ the previous expression can be rewritten as 
\begin{equation}
    \nabla_\theta\mathbb{E}\left[G_T\mid S_{T-1}\right]  =\mathbb{E}\left[G_T\nabla_{\theta}\log\pi_\theta(A_{T-1}\mid S_{T-1})\mid S_{T-1}\right]. 
\end{equation}

Since $S_1,\ldots, S_{T-1}$ are measurable with respect to $S_{T-1}$ it follows that 
\begin{align}
   &\nabla_\theta\mathbb{E}\left[G_T\mid S_{T-1}\right] \prod_{t=1}^{T-1}\mathbbm{1}\left(S_{t}\in\mathcal{S}_{\text{safe}}\right) \nonumber \\
   &=\mathbb{E}\left[G_1\nabla_{\theta}\log\pi_\theta(A_{T-1}\mid S_{T-1})\mid S_{T-1}\right],
\end{align}
where we have used that $G_1 = G_T\prod_{t=1}^{T-1}\mathbbm{1}\left(S_{t}\in\mathcal{S}_{\text{safe}}\right)$. Substituting the previous expression in the left hand side of \eqref{eqn_thing_to_show} it follows that 
\begin{align}
    &\mathbb{E}\left[\nabla_\theta\mathbb{E}\left[G_T\mid S_{T-1}\right]\prod_{t=1}^{T-1}\mathbbm{1}\left(S_{t}\in\mathcal{S}_{\text{safe}}\right) \mid S_0\right] \nonumber \\
    &=\mathbb{E}\left[\mathbb{E}\left[G_1\nabla_\theta \log \pi_{\theta}(A_{T-1} \mid S_{T-1}) \mid S_{T-1}\right]\mid S_0\right]. 
\end{align}
The law of total expectation completes the result claimed in \eqref{eqn_thing_to_show} and therefore completes the proof of Theorem~\ref{theorem_safe_policy_gradient}.

%%%%%%%%%%%%%%%%%%
%%%%%%%%%%%%%%%%%%
%%%%%%%%%%%%%%%%%%
%%%%%%%%%%%%%%%%%%
\subsection{Technical Lemmas used in Section~\ref{Losses_of_Imposing_Relaxation}}
\label{appendix_lemma_primal_error}
\begin{lemma}
\label{lemma_primal_error}
Consider $g(\cdot)$ as defined in \eqref{eqn_def_g}.
$\forall \lambda^1, \lambda^2 \in \Real$, denote by $\theta^*(\lambda^1)$ and $\theta^*(\lambda^2)$ the optimum of the Lagrangian in \eqref{eqn_dual_funtion2} with respect to $\lambda^1$ and $\lambda^2$. Suppose that there exists a $\epsilon > 0$ and a $\theta^{\dagger} (\lambda^1)$ such that
\begin{align}\label{eqn_primaL_err}
    L(\theta^\dagger(\lambda^1), \lambda^1) \leq L(\theta^*(\lambda^1), \lambda^1) \leq L(\theta^\dagger(\lambda^1), \lambda^1) + \epsilon.
\end{align}
% Let $g(\theta^{\dagger}(\lambda^1)) = \mathbb{P} \left(\cap_{t=0}^{T} \{ S_t \in \mathcal{S}_\text{safe}\} | \theta^{\dagger}(\lambda^1) \right)  -  (1-\delta)$. 
Then, it holds that
\begin{align}
    d(\lambda^2) \geq d(\lambda^1) + (\lambda^2 - \lambda^1) g(\theta^\dagger(\lambda^1)) - \epsilon.
\end{align}
\end{lemma}
\begin{proof}
We proceed by recalling the definition of the dual function $d(\lambda^1)$ and $d(\lambda^2)$ and computing their difference
\begin{align}\label{eqn_dual_difference}
    d(\lambda^1) - d(\lambda^2) = L(\theta^*(\lambda^1), \lambda^1) - L(\theta^*(\lambda^2), \lambda^2),
\end{align}
where the Lagrangian $L(\theta(\lambda), \lambda)$ is defined by \eqref{eqn_dual_funtion2}.

Since $\theta^*(\lambda^2)$ is the optimal, \eqref{eqn_dual_difference} can be converted to
\begin{align}
    d(\lambda^1) - d(\lambda^2) \leq L(\theta^*(\lambda^1), \lambda^1) - L(\theta^\dagger(\lambda^1), \lambda^2).
\end{align}
Substituting \eqref{eqn_primaL_err} into the previous inequality yields
\begin{align}
    d(\lambda^1) - d(\lambda^2) \leq L(\theta^\dagger(\lambda^1), \lambda^1) + \epsilon - L(\theta^\dagger(\lambda^1), \lambda^2).
\end{align}
Writing $L(\theta^\dagger(\lambda), \lambda)$ as in \eqref{eqn_dual_funtion2}, the previous inequality yields
\begin{align}\label{eqn_dlam1_dlam2}
    d(\lambda^1) - d(\lambda^2) &\leq V(\theta^\dagger(\lambda^1)) + \lambda^1 g(\theta^\dagger(\lambda^1)) \nonumber \\
    &- \left(V(\theta^\dagger(\lambda^1)) + \lambda^2 g(\theta^\dagger(\lambda^1)) \right) + \epsilon \nonumber \\
    &=-(\lambda^2 - \lambda^1) g(\theta^\dagger(\lambda^1)) + \epsilon.
\end{align}
Reordering \eqref{eqn_dlam1_dlam2} completes the proof of Lemma~\ref{lemma_primal_error}.
\end{proof}
\begin{lemma}\label{lemma_lam_bounded}
        Assume that there exists a strictly feasible policy $\tilde{\pi}_\theta$ such that $\exists \, C > 0, g(\tilde{\pi}_\theta) \geq C$, and $V(\tilde{\pi}_\theta)$ is bounded. Let hypotheses of Theorem~\ref{proposition_dual_converge} hold. Then, it holds that
\begin{align}
    \limsup_{k\to \infty} 1/k \, \mathbb{E}\left[\lambda^k\right] = 0.
\end{align}
\end{lemma}

\begin{proof}
Let $\tilde{\pi}_\theta$ be a strictly feasible policy such that $g(\tilde{\pi}_\theta) \geq C > 0$. Then, by the definition of dual function in \eqref{eqn_dual_funtion2}, one can lower bound $d(\lambda)$ as
    \begin{align}
        d(\lambda) &\geq V(\tilde{\pi}_\theta) + \lambda g(\tilde{\pi}_\theta)  \geq V(\tilde{\pi}_\theta) + \lambda C.
    \end{align}
The previous inequality is equivalent to
    \begin{align}\label{eqn_any_lam_bound}
         \lambda \leq \frac{d(\lambda) - V(\tilde{\pi}_\theta)}{C}, \, \forall \lambda \in \Real.
    \end{align}
Since $d(\lambda^*)$ i.e., $D^*$ and $V(\tilde{\pi}_\theta)$ are bounded, \eqref{eqn_any_lam_bound} indicates that $\lambda^*$ is bounded as well.

We next derive the following inequality using \eqref{eqn_convexity1} and the convexity of the dual function~{\cite[Chapter 5.2]{boyd2004convex}}
\begin{align}\label{eqn_bounded_expDual}
    \mathbbm{E} \left[ d\left(\frac{1}{k} \sum_{u=0}^{k-1} \lambda^u\right) \right]
    \leq &D^* + \frac{\mathbbm{E} \left[ (\lambda^{0}- \lambda^* )^2  \right]}{2 \eta_{\lambda} k}  + \frac{\eta_{\lambda} (1-\delta)^2}{2} \nonumber \\
    &+\epsilon.
\end{align}
Since the rightmost of \eqref{eqn_bounded_expDual} is bounded, we obtain 
\begin{align}\label{eqn_bounded_expDual2}
    \limsup_{k\to\infty} \mathbb{E}\left[ d(\frac{1}{k}\sum_{u=0}^{k-1}\lambda^u ) \right] < \infty.
\end{align} 
Combining \eqref{eqn_bounded_expDual2} with \eqref{eqn_any_lam_bound} further reveals that 
\begin{align}
    \limsup_{k\to\infty} \mathbb{E}\left[ \frac{1}{k}\sum_{u=0}^{k-1}\lambda^u \right] < \infty.
\end{align}

By virtue of the Stolz–Cesàro Theorem~\cite{choudary2014real}, it follows that 
\begin{align}
    \limsup_{k\to\infty} \mathbb{E}\left[ \lambda^k \right] < \infty.
\end{align} 
Ultimately, dividing by $1/k$
completes the proof of the result.
\end{proof}